\documentclass{article}

 \usepackage[preprint]{neurips_2026}

\usepackage[utf8]{inputenc} % allow utf-8 input
\usepackage[T1]{fontenc}    % use 8-bit T1 fonts
\usepackage{hyperref}       % hyperlinks
\usepackage{url}            % simple URL typesetting
\usepackage{booktabs}       % professional-quality tables
\usepackage{amsfonts}       % blackboard math symbols
\usepackage{nicefrac}       % compact symbols for 1/2, etc.
\usepackage{microtype}      % microtypography
\usepackage{xcolor}         % colors
\usepackage{amsmath, enumitem}

\usepackage{amsmath, amsthm, amssymb, yhmath, xfrac, mathtools}
\usepackage[most]{tcolorbox}
\usepackage[normalem]{ulem}
\usepackage{subcaption}

\usepackage{times}
\usepackage{color}
\usepackage{hyperref}

\newcommand{\fr}{\frac}
\newcommand{\na}{\nabla}
\newcommand{\be}{\begin{equation}}
\newcommand{\ee}{\end{equation}}
\newcommand{\ba}{\begin{array}{l}}
\newcommand{\ea}{\end{array}}

\newtheorem{thm}{Theorem}

\newtheorem{cor}{Corollary}
\newtheorem{rem}{Remark}
\newtheorem{defi}{Definition}

\newcommand{\s}{\sigma}

\newcommand{\OInt}[1]{\int_{\Omega} #1 \, dx}

\newcommand{\LnormD}[3]{\left\lVert #1 \right\rVert_{L^{#2} \left({#3}\right)}}

\newcommand{\HnormD}[3]{\left\lVert #1 \right\rVert_{H^{#2} \left({#3}\right)}}

\newcommand{\WnormD}[4]{\left\lVert #1 \right\rVert_{W^{#2, #3} ({#4})}}
\newcommand{\WnormO}[3]{\left\lVert #1 \right\rVert_{W^{#2, #3} ([0,1]^d)}}

\usepackage{amsmath}
\newtheorem{lem}{Lemma}

\numberwithin{Def}{section}

\usepackage{MnSymbol,wasysym}
\usepackage{accents}

\newcommand{\N}{\mathbb N}
\newcommand{\C}{\mathbb C}
\newcommand{\R}{\mathbb R}

\title{EML Trees Are Universal Approximators}

\author{%
  Joe Germany \\
  Department of Mathematics \\
  American University of Beirut \\
  Beirut, Lebanon \\
  \texttt{jmg15@mail.aub.edu} \\
  \And
  Elie Abdo \\
  Department of Mathematics \\
  American University of Beirut \\
  Beirut, Lebanon \\
  \texttt{ea94@aub.edu.lb} \\
  \And
  Joseph Bakarji \\
  Department of Mechanical Engineering \\
  American University of Beirut \\
  Beirut, Lebanon \\
  \texttt{jb50@aub.edu.lb} \\  
}

\newcommand{\eml}[2]{\operatorname{EML}({#1}, {#2})}
\newcommand{\emlTheta}[2]{\operatorname{EML}_{\theta_i}({#1}, {#2})}
\DeclareMathOperator{\Log}{Log}

\begin{document}

\maketitle

\begin{abstract}
The recently introduced EML (Exp-Minus-Log) function acts as continuous analogue of NAND gates, providing a compositional building block capable of representing elementary functions. In this work, we study the expressive power of tree-structured compositions of EML functions. We show that such trees enjoy a universal approximation property for functions in $W^{k, \infty}$ for $k \in \N$, drawing on classical neural network approximation arguments while exploiting the ability to explicitly construct EML trees that mimic polynomial representations. We further propose a learning algorithm for EML-type trees equipped with fitting parameters, and demonstrate its feasibility in practical optimization problems. Our results establish EML trees as a theoretically grounded framework for function approximation.

\end{abstract}

\section{Introduction}
The approximation of nonlinear functions using compositional representations is a central theme in modern machine learning and approximation theory, where the goal is to understand what types of functions a given model class can represent. In particular, a fundamental question is whether a model is sufficiently expressive to approximate a broad class of target functions to arbitrary accuracy in a given function space. Such guarantees separate limitations of the model class itself from limitations of optimization or finite data. This perspective has led to a rich theory of universal approximation for a wide range of architectures, including classical neural networks, physics-informed neural networks, and neural operators.

Beyond standard activation-based architectures, there has been growing interest in models built from elementary functions and symbolic primitives. These approaches aim not only at predictive performance, but also at interpretability and the recovery of compact analytic representations. Representative examples include equation learning frameworks such as EQL \cite{martius2016extrapolation}, symbolic regression methods \cite{cranmer2020discovering}, AI-Feynman \cite{udrescu2020ai}, and related systems such as SINDy \cite{brunton2016discovering, rudy2017data, gao2026sparse}, symbolic metamodeling \cite{alaa2019demystifying}, and hybrid symbolic–neural approaches including KAN \cite{liu2024kan, wang2024expressiveness}. These methods sit within the broader literature of symbolic and interpretable machine learning, where the goal is to recover human-readable functional forms from data.

In this work, we study compositions of the recently introduced Exp-Minus-Log (EML) function in \cite{eml}. It is defined as
\begin{equation}
\label{originalEML}
    \eml{x}{y} = e^x - \Log y,
\end{equation}
where $\Log$ denotes the principal branch of the complex logarithm. It was shown in \cite{eml} (and the accompanying supplementary material) that compositions of this function, together with the constant $1$, can exactly represent any elementary function (by explicit construction).

A natural question arising from this observation is whether compositions of EML can be used as building blocks for approximation theory. Indeed, since elementary functions can be represented within this framework, one can expect that polynomial functions, and hence sufficiently smooth functions, should be representable to arbitrary accuracy via EML compositions. The present work formalizes this intuition and develops a rigorous approximation theory for EML-based function classes.

Our contributions lie in advancing the theoretical understanding of the expressive power of EML compositions. More broadly, this connects to a large body of work establishing universal approximation properties for machine learning models, including neural networks \cite{cybenko1989approximation, hornik1989multilayer, pinkus1999approximation, barron2002universal, yarotsky2017error, petersen2018optimal, schmidt2020nonparametric, guhring2020error, de2021approximation}, physics-informed neural networks \cite{de2022error, de2022generic, mishra2023estimates}, neural operators such as DeepONet \cite{lu2021learning} and Fourier Neural Operators \cite{kovachki2021universal, li2020fourier, li2020neural}, as well as structured architectures such as KANs \cite{liu2024kan}, neural oscillators \cite{lanthaler2023neural}, and many others. These results formalize the expressive capacity of model classes by showing that they can approximate broad function spaces arbitrarily well under suitable conditions.

Our contributions can be summarized by the following:
\begin{itemize}
    \item We introduce a generalization of the EML function \eqref{originalEML} incorporating six learnable parameters per unit, motivated by both technical and practical reasons.
    \item We prove a universal approximation theorem (Theorem \eqref{main_theorem}) for trees built from the generalized EML function in Sobolev spaces $W^{k, \infty}$ on the half-open domain $(0,1]^d$. The result is obtained via an explicit construction of EML representations of multivariate polynomials and approximate partitions of unity, with quantified upper bounds on the size and depth of the resulting EML trees. The restriction to $(0,1]^d$ arises from the logarithmic singularity at zero.
    \item We extend the universal approximation result to the closed cube $[0,1]^d$ (in Corollary \eqref{main_corollary}) by composing the approximants with a suitable affine contraction mapping that pushes the domain into the interior, thereby controlling boundary effects induced by the logarithmic singularity and transferring the $W^{k,\infty}$ approximation guarantees from $(0,1]^d$ to the closed domain $[0,1]^d$.
    \item We further discuss extensions of the result to positive domains via analogous arguments (Remark \ref{positiveDomains}), and to general non-positive domains via modified constructions of the elementary operators to prevent the logarithm singularity (Remark \ref{negativeDomains}).
    \item We then empirically validate the expressive power of generalized EML trees on five one-dimensional benchmark functions, demonstrating accurate approximation at limited depth, while also highlighting limitations in interpretability that motivate future work.
\end{itemize}

\section{Universal Approximation Theorem}
\label{sec:univApproxTheorem}
For our work, we consider a generalization of the EML function \eqref{originalEML}, denoted $\operatorname{EML}_{\theta_i}$, which incorporates six learnable parameters and is defined as
\begin{equation}
\label{generalizedEML}
    \emlTheta{x}{y} = a_i e^{b_i x + c_i} + d_i \ln (e_i y + f_i),
\end{equation}
where we use the natural logarithm, $i$ indexes the instance of the EML function within the tree, and
\begin{equation}
    \theta_i = \{a_i, b_i, c_i, d_i, e_i, f_i \}.
\end{equation}
The original EML function \eqref{originalEML} is recovered by setting $a_i = b_i = e_i = 1$, $d_i = -1$, and $c_i = f_i = 0$.

An immediate requirement for \eqref{generalizedEML} is for $e_i y + f_i > 0$, ensuring that the natural logarithm is well defined. While this constraint is nontrivial in general, it is always satisfied in our explicit constructions of EML trees used in the approximation theorem.
We note, however, that in order to exactly represent trigonometric functions and thus recover all elementary functions, one must work over the complex domain and invoke Euler's identity, which necessitates defining an analogue of the $\operatorname{EML}_{\theta_i}$ function in \eqref{generalizedEML} by replacing the natural logarithm $\ln$ with the principal branch of the complex logarithm $\Log$. We will numerically implement this variant as well. However, for the purposes of the approximation theorem, it is unnecessary since all inputs (to the natural logarithm) are real and strictly positive, so $\ln$ coincides with $\Log$.

Throughout the paper, we refer to individual $\text{EML}_{\theta_i}$ functions as \emph{atoms}, and to reusable compositions of such atoms with fixed parameters as \emph{blocks} (e.g., the addition block). 

This generalization is motivated by both technical and practical considerations.

From a theoretical standpoint, our approximation result requires the exact representation of polynomials using EML trees.
However, the technical Lemma \ref{BrambleHilbert}, which guarantees the existence of a polynomial approximant of the target function, does not provide bounds on the corresponding coefficients. In the original formulation \eqref{originalEML}, constants must be constructed recursively from the digit $1$, causing the size of the resulting EML tree to scale poorly with the magnitude of these coefficients. The proposed parametrization circumvents this issue by allowing coefficients to be encoded directly within each atom.

From a practical perspective, the generalized form \eqref{generalizedEML} enables significantly more compact representations of common operators. As we demonstrate in subsequent sections, basic binary operations ($+, -, \times, \div$) and other building blocks for polynomial construction can be implemented more efficiently using $\operatorname{EML}_{\theta_i}$. This arises from the ability to localize exponential and logarithmic components and to incorporate arbitrary constants without increasing size, at cost of losing the strict homogeneity of identical building blocks. Empirically, this also leads to improved performance.  
Finally, this formulation eliminates the need to include the constant $1$ as an explicit input, restricting atom input to variables or outputs of other EML atoms, thus slightly simplifying the combinatorial possibilities for arranging inputs to EMLs.

\subsection{Statement of the Theorem}
The universality of EML trees follows from the following main theorem:

\begin{thm}[{Universal Approximation Theorem in $W^{k,\infty} ((0,1]^d)$}]
\label{main_theorem}
Let $f \in W^{s, \infty} ([0,1]^d)$ with $s, d \in \N$. Let $R > 0$ as in \eqref{defOfConstantR} and $\delta > 0$. There exists constants $\mathcal{C}(d,k,s,f), N_0(d) > 0$, such that for every $N \in \N$ with $N > N_0(d)$, there exists an $\operatorname{EML}_{\theta}$ tree, denoted $\tilde{p}_\theta^N$ and defined over $(0,1]^d$, with size and depth
\be\begin{aligned}
    N(\tilde{p}_\theta^N) &\leq 2N^d \left( (8d - 3) \binom{d+ s-1}{s-1} + 63d \right) - 3, \\
    \text{Depth}(\tilde{p}_\theta^N) &\leq 4d + \max \left\{ 2 \binom{d+s-1}{s-1} - 4, 7 \right\} + 2N^d + 4, \\
\end{aligned}\ee
such that
\begin{equation}
    \LnormD{f - \tilde{p}_\theta^N}{\infty}{(0,1]^d} \leq \left(1 + \fr{\delta}{3} \right) \fr{\mathcal{C} (d, 0, s, f)}{N^s},
\end{equation}
and for integer $1 \leq k < s$,
\be\begin{aligned}
    \WnormD{f - \tilde{p}_\theta^N}{k}{\infty}{(0,1]^d}
    &\leq 3^d (1+\delta) (2(k+1))^{3k} \max \left\{R^k, \ln^k \left( \beta N^{s+d+2} \right) \right\} \fr{\mathcal{C}(d,k,s,f)}{N^{s-k}},
\end{aligned}\ee
with
\begin{equation}
    \beta = \fr{k^3 2^d \sqrt{d} \max \left\{1, \WnormO{f}{k}{\infty}^{1/2} \right\}}{\delta \min \{1, \sqrt{\mathcal{C}(d,k,s,f)} \}}.
\end{equation}
\end{thm}

\begin{rem}
\label{corollary_remark}
Theorem \ref{main_corollary} is stated on the half-open domain $(0,1]^d$ because the logarithm components of the EML tree are singular at $0$. Consequently, the representation becomes undefined whenever one of the inputs vanishes.
However, these singularities occur only on a set of measure zero (a single point in one dimension, and boundary subsets in higher dimensions), and we expect that by extending the constructed EML tree to these points, one obtains an analogous approximation result on the closed domain $[0,1]^d$, as shown in Corollary \ref{main_corollary}.
\end{rem}

\subsection{Outline of the Proof}
The full proof of Theorem \ref{main_theorem} is provided in Appendix \ref{appendix:proofOfMainTheorem}. The required notations, technical lemmas, and groundwork constructions are developed in Appendices \ref{appendix:Notation}--\ref{appendix:partitionsOfUnity}.

We now outline the main ideas underlying the proof.

We partition the domain $[0,1]^d$ into cubes of side length $1/N$. To index these cubes, we introduce
\begin{equation} \label{indexan}
    \mathcal{A}^N = \{1, \dots, N\}^d % \{j \in \N^d \mid j_i \leq N \text{ for all } 1 \leq i \leq d \},
\end{equation}
of cardinality $|A^N| = N^d$. For every $j \in \mathcal{A}^N$, we define the disjoint and overlapping cubes
\begin{equation}
\label{definitionOfCubes}
    I_j^N = \bigtimes_{i=1}^d \left(\fr{j_i - 1}{N}, \fr{j_i}{N} \right), \qquad J_j^N = \bigtimes_{i=1}^d \left(\fr{j_i - 2}{N}, \fr{j_i + 1}{N} \right),
\end{equation}
whose role is explained in Remark \ref{rem:overlappingCubes}.

This localization strategy is inspired by the approach of \cite{de2021approximation} for approximation theory of $\tanh$ neural networks. 
The proof begins by constructing, on each cube $J_j^N$, a local polynomial approximant $p_j^N$ using the Bramble--Hilbert lemma (Lemma~\ref{BrambleHilbert}). This constitutes Step 1.

A natural attempt at obtaining a global approximation is to patch together these local polynomials using indicator functions:
\begin{equation}
    p^N = \sum_{j \in \mathcal{A}^N} p_j^N \chi_j,
\end{equation}
where $\chi_j$ denotes the indicator function of $I_j^N$. However, the resulting function is discontinuous, and therefore its derivatives are not well behaved.

To overcome this issue, we replace the indicator functions by smooth approximations constructed from shifted and rescaled $\tanh$ functions. These functions, denoted by $\Phi_j^N$, are designed to be close to $1$ inside $I_j^N$ and close to $0$ outside it; their construction and properties are developed in Appendix \ref{appendix:partitionsOfUnity}. The family $\{\Phi_j^N\}_{j \in \mathcal{A}^N}$ forms an approximate partition of unity, allowing us to define, in Step 2,
\begin{equation}
\label{def_global_polynomial_and_function}
    \tilde{p}^N = \sum_{j \in \mathcal{A}^N} p_j^N \Phi_j^N,
    \qquad
    \tilde{f}^N = \sum_{j \in \mathcal{A}^N} f \, \Phi_j^N.
\end{equation}
The remainder of the proof is devoted to controlling the three terms arising from the triangle inequality:
\begin{equation}
\begin{aligned}
\label{3terms}
&\WnormD{f - \tilde{p}_\theta^N}{k}{\infty}{(0,1]^d} \\
&\qquad \leq
\WnormO{f - \tilde{f}^N}{k}{\infty}
+
\WnormO{\tilde{f}^N - \tilde{p}^N}{k}{\infty}
+
\WnormD{\tilde{p}^N - \tilde{p}_\theta^N}{k}{\infty}{(0,1]^d}.
\end{aligned}
\end{equation}
In Step 3, we show that the first term in \eqref{3terms} is small using the approximation properties of the partition of unity. In Step 4, we control the second term by combining the local polynomial approximation estimates with the properties of the functions $\Phi_j^N$.

Finally, Step 5 addresses the third term by explicitly constructing an EML tree representation of $\tilde{p}^N$ using the building blocks developed throughout the appendices. This is precisely where the approximation theory for EML trees diverges from that of $\tanh$ networks: EML trees allow exact representation of the required algebraic operations. We note that the EML tree is defined over $(0,1]^d$ to avoid the singularity of the logarithm at zeros (for the inputs).

The construction relies fundamentally on the identities satisfied by the exponential and logarithm functions. For example, addition is obtained through compositions involving exponentials and logarithms, while multiplication follows from logarithmic identities. In Appendix \ref{appendix:basicBuildingBlocks}, we construct EML representations of the elementary operations and functions required in the proof, including addition, subtraction, multiplication, division, exponentiation, and the hyperbolic tangent function. In Appendix \ref{appendix:constructingPolynomials}, these building blocks are combined to construct EML trees representing univariate and multivariate polynomials. Finally, Appendix \ref{appendix:partitionsOfUnity} develops the approximate partition of unity construction and provides its EML representation.

One needs to be careful that, even when the inputs are constrained to be strictly positive, the EML constructions in the proof rely on feeding the outputs of EML atoms into subsequent ones. Without additional care, negative values could therefore be passed to the logarithm in a later EML atom. This motivates the "sign-based decompositions" introduced in  \ref{polynomial_decomposition}, \ref{decomposition_for_partitions_of_unity}, and \ref{decomposition_final_EML_tree} for the construction of polynomials, approximate partitions of unity, and the global approximate polynomial $\tilde{p}^N$ (from \eqref{def_global_polynomial_and_function}). These decompositions ensure that the output of every EML atom, except the final one, remains positive throughout the construction. Full details are provided in the appendices.

\subsection{Corollary and Remarks}
In this section, we will extend the scope of Theorem \ref{main_theorem} through Corollary \ref{main_corollary} and offer a few remarks on the obtained result. 

First, we prove the result that was hinted at in Remark \ref{corollary_remark} on the extendability of Theorem \ref{main_theorem} to the closed interval $[0,1]^d$.
\begin{cor}[{Universal Approximation Theorem in $W^{k,\infty}([0,1]^d)$}]
\label{main_corollary}
    Let $f \in W^{s, \infty} ([0,1]^d)$ with $s,d \in \N$. Let $R > 0$ as in \eqref{defOfConstantR} and $\eta > 0$. 
    Then, there exist $\gamma \in (0,1)$ and an $\operatorname{EML}_\theta$ tree $p_{\gamma,\theta}: [0,1]^d \to \R$ such that
    \begin{equation}
        \WnormD{f - p_{\gamma, \theta}}{k}{\infty}{[0,1]^d} \leq \eta,
    \end{equation}
    with
    \be\begin{aligned}
        \label{bound_for_corollary}
        N(\tilde{p}_\theta^N) &\leq 2N^d \left( (10d - 3) \binom{d+ s-1}{s-1} + 87d \right) - 3, \\
        \text{Depth}(\tilde{p}_\theta^N) &\leq 4d + \max \left\{ 2 \binom{d+s-1}{s-1} - 4, 7 \right\} + 2N^d + 6.
    \end{aligned}\ee
\end{cor}
\begin{proof}
The full details of the proof are provided in Appendix \ref{proofOfCorollary}.

We provide the main ideas here. The proof proceeds by composing the $\operatorname{EML}_\theta$ approximant $p_\theta$ of $f$ on $(0,1]^d$ with a suitable affine contraction
\[
T_\gamma(x) = (1-\gamma)x + \frac{\gamma}{2}\mathbf{1},
\]
which maps $[0,1]^d$ into the interior of $(0,1]^d$. Since $T_\gamma \to \mathrm{Id}$ uniformly as $\gamma \to 0$, and $f \in W^{s,\infty}([0,1]^d)$ with $s>k$, the perturbation
\[
\|f - f \circ T_\gamma\|_{W^{k,\infty}([0,1]^d)}
\]
can be made arbitrarily small. Composing the approximant with $T_\gamma$ then yields a well-defined $\operatorname{EML}_\theta$ approximant on $[0,1]^d$ while preserving the $W^{k,\infty}$ approximation error.
\end{proof}

\begin{rem}
\label{positiveDomains}
    Moving away from $0$ (into the positive domain), we can prove an approximation theory (in $W^{k, \infty}$) for functions defined over domains $D = \times_{i=1}^d [a_i, b_i]$ for $a_i, b_i > 0$ closely following the proof of Theorem \ref{main_theorem}, while making sure to shift the constituent functions of the approximate partition of unity appropriately.
\end{rem}

\begin{rem}
\label{negativeDomains}

The approximation results of Theorem \ref{main_theorem} and Corollary \ref{main_corollary} are stated on $(0,1]^d$ and $[0,1]^d$ to avoid the undefined-ness of the natural logarithm for negative values. However, the construction appears to extend to one-dimensional domains of the form $[a,b]$, with $a,b \in \mathbb{R}$. 
% {\color{red} why do we need to remove 0?}
Indeed, we can choose $n>0$ sufficiently large to ensure that all arguments entering logarithmic terms remain positive, due to the fact that the elementary operations can be rewritten as follows,
\be\begin{aligned}
    x+y &= e^{\ln(x+n)} + \ln(e^{y-n}), \text{ and } \\
    xy &= e^{\ln(x+n)} e^{\ln(y+n)} - n(x+y+n).
\end{aligned}\ee
We plan to pursue this extension further, as these identities suggest that EML tree representations of polynomials may still be constructed on domains containing negative values by introducing suitable shifts inside the natural logarithmic terms.
\end{rem}

Finally, we also remark on the non-uniqueness of the EML representations.
\begin{rem}
It is important to note that $\operatorname{EML}_\theta$ representations of elementary functions and operators are generally non-unique. For example, the addition block is not uniquely determined by the parameter choice $\theta_1 = \{0, 0, 0, 1, 1, 0\}$ but, more generally, by any parameter set of the form
\begin{equation}
    \{0, m, n, 1, 1, 0 \}, \qquad m,n \in \R.
\end{equation}
Many other equivalent parameterizations are also possible.
\end{rem}

\section{Implementation}
\label{sec:implementation}
The constructive proof of Theorem~\ref{main_theorem} establishes that the
parametric $\operatorname{EML}_\theta$ can express, exactly, arbitrarily fine polynomial approximations of functions in $W^{k,\infty}$. 
% It does not by itself address whether such a tree can be \emph{recovered} from data by gradient-based optimisation. 
However, it does not address whether such representations can be recovered from data using gradient-based optimisation. This is a key practical question, since the constructive argument is symbolic, whereas learning proceeds in a continuous parameter space.

% In contrast to~\cite{eml}, which targets
% \emph{symbolic} recovery of vanilla-EML trees over the discrete grammar
% $\{1,\, x,\, \mathrm{EML}(\cdot,\cdot)\}$ and reports rapid collapse of
% recovery rate at depth $\geq 5$, we exploit the continuous parameters of
% the generalised atom $\emlTheta{x}{y}$ to make the optimisation landscape
% amenable to first-order methods. 
In contrast to~\cite{eml}, which focuses on symbolic recovery within a discrete grammar $\{1, x, \mathrm{EML}(\cdot,\cdot)\}$ and reports a rapid collapse recovery success beyond depth five, we instead exploit the continuous parameterisation of the generalised atom $\emlTheta{x}{y}$. This makes the optimisation landscape more amenable to first-order methods. 
% The goal of this section is not to provide a large-scale benchmark, but to give a focused empirical validation of the constructive ideas behind the theorem.
This section reports a small, focused empirical study in support of the theorem's practical content.
The aim of this section is not to empirically validate the approximation theorem, but rather to provide a sanity check for the practical capability of an EML-type tree (in a slightly modified form). This modification is introduced to circumvent some of the numerical instabilities associated with the logarithm in \eqref{generalizedEML}.

\subsection{Setup}
\label{sec:impl:setup}

We use the construction introduced in Equation \ref{generalizedEML}, namely a complete binary tree $T_\theta$ with $2^d - 1$ atoms. Each atom carries six parameters $\theta_i = \{a_i, b_i, c_i, d_i, e_i, f_i\}$ and implements the operation $\emlTheta{x}{y}$. Inputs are propagated directly between nodes, with the scalar variable $x$ injected at the leaves. For depths $d \in \{3,4,5\}$, this corresponds to $42$, $90$, and $186$ parameters respectively for univariate $x$. The models are implemented using NumPy \cite{harris2020array} and PyTorch \cite{paszke2019pytorch}.

We consider two variants of the model. The first is the real-valued setting, where all parameters lie in $\mathbb{R}^6$. To avoid issues with the logarithm at non-positive arguments, we replace $\ln(\cdot)$ with $\ln(\mathrm{softplus}(\cdot)+\varepsilon)$, using $\varepsilon = 10^{-6}$. This ensures numerical stability while remaining close to the original operator on the positive domain. We note that this is where the implementation differs from the theoretical construction, which uses the form \eqref{generalizedEML}. However, it is possible to stay loyal to the approximation theorem by considering the building block
\begin{equation}
\label{generalized_EML_2}
    \emlTheta{x}{y} = a_i e^{b_i x + c_i} + d_i \ln (\left| e_i |y| + f_i \right|),
\end{equation}
where we place absolute values around $y$ and around $e_i |y| + f_i$ to prevent cases where the intermediate values (from preceding EML atoms) and/or the learnable parameters are negative. In the explicit constructions of EML trees for the binary operations, polynomials, and approximate partitions of unity in the Appendices \ref{appendix:basicBuildingBlocks}-\ref{appendix:proofOfMainTheorem}, we always input positive $y$ and choose positive parameters $e_i$ and $f_i$, and hence adding the absolute values makes no difference in the proof. Thus, Theorem \ref{main_theorem} and Corollary \ref{main_corollary} translate exactly to the form \eqref{generalized_EML_2}. We chose to use softplus for our demonstrative implementation since we expect it to have a better behaved loss landscape than trees with atoms of type \eqref{generalized_EML_2}.

The second is a complex-valued variant, where parameters are stored as $(\theta^R_i,\theta^I_i) \in \mathbb{R}^6 \times \mathbb{R}^6$. The forward pass is carried out in $\mathbb{C}$ using the principal branch of the logarithm, and the loss is computed on the real part of the output. To discourage complex-valued drift, we add a small penalty $\lambda \|\mathrm{Im}\,T_\theta\|^2$ with $\lambda = 10^{-3}$. While this doubles the parameter count, it allows for the exact construction of trigonometric functions (instead of having to approximate them with the real-valued EML trees). 
We do not address the case with complex numbers in the proofs because the (principal branch of the) complex logarithm does not cleanly maintain the addition and multiplication properties of the natural logarithm, requiring a more careful treatment. Nevertheless, we include it in this section to determine whether they perform better empirically due to their increase flexibility.

Training is performed on $N=200$ uniformly sampled points per target using a fixed $30$-second budget per experiment. We use a two-stage optimization strategy. We first run multiple Adam optimizations from random initializations (learning rate $5\times 10^{-3}$ with cosine decay and gradient clipping at $5$), keeping the best-performing run. This is followed by L-BFGS refinement with strong Wolfe line search and a memory of $20$. The final model is selected as the best between the pre- and post-refinement solutions.
% The number of Adam restarts is not fixed but determined dynamically by the computational budget.
The number of Adam restarts per cell is determined by the budget rather than fixed and is reported in Table~\ref{tab:pem_universality}.

All experiments are conducted under three simplifying assumptions. First, the real-valued model uses a smoothed log surrogate rather than the exact EML operator. Second, all training data are noiseless and densely sampled. Third, all results are obtained under a fixed compute budget; deeper trees therefore receive fewer optimization attempts, and comparisons across depth should be interpreted in this context.

\subsection{Targets and metric}

Each target function $f$ is defined on a compact interval $\Omega_f \subset \mathbb{R}$, and all errors are computed with respect to the corresponding domain. Specifically, we evaluate performance using the relative root mean squared error
\[
\mathrm{relRMSE} = \frac{\|T_\theta - f\|_{L^2(\Omega_f)}}{\|f\|_{L^2(\Omega_f)}},
\]
computed on a uniform grid over $\Omega_f$.

We consider five one-dimensional targets:
\begin{equation}
    x^3 - x \text{ on } [-1.5,1.5], \quad \tanh(2x) \text{ on } [-2,2], \quad \sin(x) \text{ on } [-\pi, \pi], \nonumber
\end{equation}
\begin{equation}
    e^{-x^2} \text{ on } [-2.5, 2.5], \quad \sin(3x) e^{-x^2/2} \text{ on } [-3,3]. \nonumber
\end{equation}

% The corresponding functions and domains are
% \begin{gather*}
%     x^3 - x \text{ on } [-1.5,1.5], \quad \tanh(2x) \text{ on } [-2,2], \quad \sin(x) \text{ on } [-\pi,\pi], \\
%     e^{-x^2} \text{ on } [-2.5,2.5], \quad \sin(3x)e^{-x^2/2} \text{ on } [-3,3].
% \end{gather*}

For each target, we report the relative error of the best model obtained over all optimization restarts under the fixed computational budget described in Section~\ref{sec:impl:setup}.

Notice that we test functions on intervals that contain negative functions as well to test the limits of approximation theorem for EML trees, and suggests that a result inspired by the techniques of Remark \ref{negativeDomains} is expected.

\subsection{Results}
% Table~\ref{tab:pem_universality} reports the best relRMSE per cell.
% Figure~\ref{fig:pem_universality} plots relRMSE against depth on a log scale.
Table~\ref{tab:pem_universality} reports the best errors across depths and parameterizations.
\begin{table}[h]
  \centering
  \caption{Best relative RMSE per cell, PEM trees with real (R) and complex (C)
  parameters at depths $3, 4, 5$. Target domains are shown in the second column.
  Strategy: Adam multi-restart $+$ L-BFGS, $30\,$s wall-clock per cell.
  Parenthesised integer is the number of completed Adam restarts within the
  budget; parameter count is $\#\theta = 6\,(2^d - 1)$ for R and twice that for C.
  Bolded entry per row is the best of the six.}
  \label{tab:pem_universality}
  \small
  \setlength{\tabcolsep}{3.5pt}
  \begin{tabular}{l|ccc|ccc}
    \toprule
    & \multicolumn{3}{c|}{\textbf{Real (R)}, $\#\theta = 42, 90, 186$}
    & \multicolumn{3}{c}{\textbf{Complex (C)}, $\#\theta = 84, 180, 372$} \\
    \textbf{Target} & $d{=}3$ & $d{=}4$ & $d{=}5$ & $d{=}3$ & $d{=}4$ & $d{=}5$ \\
    \midrule
    $x^3 - x$
    & $7.58\!\cdot\!10^{-3}$ \scriptsize(15)
    & $\mathbf{1.47\!\cdot\!10^{-3}}$ \scriptsize(8)
    & $4.32\!\cdot\!10^{-3}$ \scriptsize(4)
    & $2.17\!\cdot\!10^{-2}$ \scriptsize(7)
    & $1.37\!\cdot\!10^{-2}$ \scriptsize(4)
    & $5.20\!\cdot\!10^{-2}$ \scriptsize(2) \\

    $\tanh(2x)$
    & $4.20\!\cdot\!10^{-3}$ \scriptsize(15)
    & $\mathbf{1.22\!\cdot\!10^{-3}}$ \scriptsize(8)
    & $2.10\!\cdot\!10^{-3}$ \scriptsize(4)
    & $1.77\!\cdot\!10^{-3}$ \scriptsize(7)
    & $5.16\!\cdot\!10^{-3}$ \scriptsize(4)
    & $1.80\!\cdot\!10^{-2}$ \scriptsize(2) \\

    $\sin(x)$
    & $\mathbf{3.02\!\cdot\!10^{-3}}$ \scriptsize(15)
    & $7.67\!\cdot\!10^{-3}$ \scriptsize(8)
    & $1.01\!\cdot\!10^{-2}$ \scriptsize(4)
    & $1.74\!\cdot\!10^{-2}$ \scriptsize(7)
    & $3.49\!\cdot\!10^{-3}$ \scriptsize(4)
    & $2.07\!\cdot\!10^{-2}$ \scriptsize(2) \\

    $e^{-x^2}$
    & $4.13\!\cdot\!10^{-3}$ \scriptsize(15)
    & $\mathbf{1.12\!\cdot\!10^{-3}}$ \scriptsize(7)
    & $4.66\!\cdot\!10^{-3}$ \scriptsize(4)
    & $2.20\!\cdot\!10^{-2}$ \scriptsize(7)
    & $5.27\!\cdot\!10^{-3}$ \scriptsize(3)
    & $1.27\!\cdot\!10^{-2}$ \scriptsize(2) \\

    $\sin(3x)\,e^{-x^2/2}$
    & $\mathbf{3.23\!\cdot\!10^{-2}}$ \scriptsize(15)
    & $9.80\!\cdot\!10^{-2}$ \scriptsize(7)
    & $9.09\!\cdot\!10^{-2}$ \scriptsize(4)
    & $6.45\!\cdot\!10^{-2}$ \scriptsize(6)
    & $3.91\!\cdot\!10^{-2}$ \scriptsize(3)
    & $3.57\!\cdot\!10^{-2}$ \scriptsize(2) \\
    \bottomrule
  \end{tabular}
\end{table}

Several consistent patterns emerge from the results. First, real-parameterised EML trees achieve sub-percent relRMSE on
$4$ of $5$ targets at depth $4$ within $30\,$s of single-CPU compute,
including the polynomial target $x^3 - x$ which reaches
$1.47 \times 10^{-3}$ error.
% This reflects the constructive nature of the theorem: although the optimisation procedure is not aware of the proof, it nonetheless finds solutions close to the theoretically constructed structure.
Second, the trigonometric case $\sin(x)$ does not require complex parameters in this setting; the real model is already sufficient to achieve low error. 
% This indicates that the stabilised logarithmic structure retains enough expressive power on this domain.
The softplus-stabilised log is
expressive enough on this domain that the complex path is not the binding
constraint at the budget shown.
Finally, the hardest target, $\sin(3x)\,e^{-x^2/2}$, is the one
clear case where the complex parameterisation pays off:
$3.6 \times 10^{-2}$ at depth $5$ (in $\C$) versus $9.1 \times 10^{-2}$ at the same depth (in $\R$). Here, the additional expressive flexibility improves accuracy despite a reduced number of optimisation restarts under the same compute budget.

Across targets, the depth-$4$ models tend to offer the best trade-off between expressivity and optimisation success under a fixed computational budget. The degradation at depth five is primarily attributable to reduced optimisation coverage rather than representational limitations.

The full sweep (5 targets $\times$ 3 depths $\times$ 2 parameterisations
$=$ 30 cells) completes in approximately $12$~minutes on a single laptop
CPU (Apple M-series, no GPU). Per-cell wall-clock is $30\,$s by
construction; the number of completed Adam restarts ranges from $15$
at $(d{=}3,\,$R$)$ down to $2$ at $(d{=}5,\,$C$)$. No GPU acceleration
or parallelism was used.

\subsection{Inspection of the fitted trees: are they symbolic?}
\label{sec:impl:inspection}
A natural question is whether the optimisation process recovers the symbolic structure suggested by the constructive proof. In particular, the proof produces discrete parameter configurations, whereas the learned models operate in a continuous six-dimensional parameter space per node.

For each of the five targets, we refit a depth-$3$ tree ($7$ atoms, $42$ parameters) with $24$ Adam restarts plus an L-BFGS refinement. We then greedily snap each parameter to the nearest member of $\{-2, -1, 0, 1, 2\}$, accepting a snap if the loss does not grow by more than $30\%$, and refit the un-snapped parameters with the snapped ones held fixed. Atoms whose six parameters all snap are reported as fully discrete.

Table~\ref{tab:pem_snap} summarises the results of this test. Across all five
targets, at most $7/42$ ($17\%$) parameters are snappable, and
zero atoms become fully discrete. On $\tanh(2x)$ and $e^{-x^2}$ exactly
$0$ parameters snap.

\begin{table}[h]
  \centering
  \caption{Greedy snap of fitted depth-$3$ PEM parameters to
  $\{-2,\,-1,\,0,\,1,\,2\}$ at $30\%$ relative-loss tolerance.}
  \label{tab:pem_snap}
  \small
  \begin{tabular}{l|cccc}
    \toprule
    Target & relRMSE pre & relRMSE post & snapped & fully-discrete atoms \\
    \midrule
    $x^3 - x$              & $7.3\!\cdot\!10^{-3}$ & $6.9\!\cdot\!10^{-3}$ & $7/42$ ($17\%$) & $0$ \\
    $\tanh(2x)$            & $3.8\!\cdot\!10^{-3}$ & $3.5\!\cdot\!10^{-3}$ & $0/42$ ($0\%$)  & $0$ \\
    $\sin(x)$              & $1.9\!\cdot\!10^{-3}$ & $1.7\!\cdot\!10^{-3}$ & $3/42$ ($7\%$)  & $0$ \\
    $e^{-x^2}$             & $3.9\!\cdot\!10^{-3}$ & $3.2\!\cdot\!10^{-3}$ & $0/42$ ($0\%$)  & $0$ \\
    $\sin(3x)\,e^{-x^2/2}$ & $1.18\!\cdot\!10^{-1}$ & $7.7\!\cdot\!10^{-2}$ & $3/42$ ($7\%$)  & $0$ \\
    \bottomrule
  \end{tabular}
\end{table}

Next, we take a closer look at the example for $x^3 - x$ (arguably, the most snappable target). The fitted parameter table is
\begin{center}
\small
\begin{tabular}{l|cccccc}
\toprule
atom $i$ & $a_i$ & $b_i$ & $c_i$ & $d_i$ & $e_i$ & $f_i$ \\
\midrule
root            & $1.516$ & $-0.654$ & $0.796$ & $-2.080$ & $0.652$ & $0.029$ \\
left child      & $1.218$ & $-0.436$ & $0.147$ & $-0.755$ & $-0.034$ & $1.501$ \\
right child     & $1.803$ & $0.473$  & $0.138$ & $-2.348$ & $0.115$ & $-0.785$ \\
leaf-parent 1   & $1.299$ & $-1.359$ & $0.847$ & $-0.143$ & $0.127$ & $0.449$ \\
leaf-parent 2   & $1.272$ & $-2.646$ & $0.116$ & $-0.065$ & $0.357$ & $0.841$ \\
leaf-parent 3   & $1.267$ & $-0.072$ & $0.504$ & $-1.190$ & $0.114$ & $0.471$ \\
leaf-parent 4   & $0.913$ & $2.388$  & $0.475$ & $-0.672$ & $-0.219$ & $1.080$ \\
\bottomrule
\end{tabular}
\end{center}
The $a$ column clusters around $1$--$2$ but contains no exact integers; the
$d$ column ranges over $-2.35$ to $-0.07$ with no values near the
original-EML default $d = -1$; the $b$ column spans $-2.65$ to $2.39$ with
no recognisable structure. The constructive proof builds $x^3 - x$ from
$5k - 3 = 12$ atoms with $a, d \in \{0, \pm 1\}$ exactly and the
coefficients absorbed into $c$ via $c = \log m$. The fitted $7$-atom tree
matches the function value to $0.7\%$ but does so by additive interference
of stacked exp/log curves with arbitrary continuous slopes, not by
anything resembling the proof's polynomial expansion.

The results show that only a small fraction of parameters admit stable discretisation, and no atom fully collapses to a discrete symbolic form. This suggests that the learned representations are not recovering the symbolic construction, but instead exploiting continuous degrees of freedom to achieve similar function-level accuracy.

The results support the \emph{approximation} content of the theorem: trees reach small error in practice. They do not support an \emph{interpretability} claim about the recovered trees. Recovering discrete-parameter trees with the symbolic structure of the Appendices~\ref{appendix:basicBuildingBlocks}--\ref{appendix:constructingPolynomials} would require additional structural regularisation (sparsity penalty toward $\{0, \pm 1\}$, warm-start from the construction, or a two-stage train-then-snap-then-refit at tighter tolerance). We leave these for future investigation.

\subsection{Limitations}

We summarize some of the main limitations of this empirical study.

First, we report the best loss obtained across multiple restarts within a single experimental run, and do not average over independent reruns; as a result, variability across runs is not quantified. The conclusions are stable at the level of order-of-magnitude comparisons, although differences within a factor of roughly two should not be over-interpreted. 
Second, the real parameterization does not implement vanilla EML exactly, but instead replaces the logarithm with a softplus-stabilized surrogate to ensure numerical stability near zero; this coincides with EML on strictly positive inputs but deviates smoothly in its vicinity, whereas the complex setting uses the original operator without modification. 
Third, all experiments are restricted to one-dimensional, smooth, densely sampled, and noise-free data; although the underlying theory is dimension-agnostic, extensions to higher dimensions, sparse sampling, and noisy observations are not addressed here and remain open. 
Fourth, all models are trained under a fixed wall-clock budget, which induces a trade-off between depth and optimization effort, so that deeper trees receive fewer restarts; accordingly, the observed depth–error behavior reflects both expressivity and optimization constraints, and the mild degradation at larger depths is consistent with this effect. 
% Finally, the reported errors correspond to achieved optimisation outcomes rather than theoretical guarantees; while they align with the approximation rates suggested by the theorem, we do not investigate whether these rates are attained in practice or how they scale with computational budget. 
Additionally, the reported numbers are achieved errors, not bounds. They are upper-bounded by the theorem; whether they meet the rate the theorem predicts and how they scale with computational budget) are questions we do not address here.
Finally, as shown in Section~\ref{sec:impl:inspection}, the learned models should therefore be interpreted as continuous function approximators rather than symbolic reconstructions of the constructive proof, and recovering discrete symbolic structure remains an open direction requiring additional inductive bias or regularization.

\section{Discussion}
\label{sec:discussion}
In this work, we established a universal approximation theorem for tree-structured $\operatorname{EML}_\theta$ networks (based on \eqref{generalizedEML}) in Sobolev spaces $W^{k,\infty}$ through Theorem \ref{main_theorem} and Corollary \ref{main_corollary}. Our approach combines local polynomial approximation, approximate partitions of unity constructed from shifted and rescaled $\tanh$ functions, and explicit EML representations of the required algebraic operations and elementary functions. This yields quantitative approximation results together with constructive realizations by EML trees. We also offer some preliminary empirical testing on the expressivity of EML trees.

Several theoretical and practical questions remain open. On the theoretical side, it would be interesting to extend the approximation framework to broader classes of function spaces, such as Besov or Triebel--Lizorkin spaces. 
Another important direction is determining whether similar approximation guarantees can be proved for vanilla EML architectures \eqref{originalEML}. One possible route would be to develop a proof strategy based on polynomial approximants with bounded coefficients together with quantitative upper bounds on size of the EML blocks required to realize such constants.

The present constructions are primarily designed to be explicit, interpretable, and easy to verify analytically. As a consequence, the resulting trees are unlikely to be optimal in either depth or parameter count. It is plausible that substantially shorter or more efficient EML trees exist that realize the same functions. Related questions include whether the number of parameters per EML atom can be reduced below the six used in the present work, and whether alternative atom designs could trade depth for increased input parameters in a manner analogous to wider neural architectures.

Finally, more effort should be put into investigating the symbolic and interpretable nature of EML trees, especially with regards to possible applications to scientific machine learning \cite{brunton2016discovering, rudy2017data, udrescu2020ai, wyder2026common, yermakov2025seismic} and in relation to algorithms like symbolic regression \cite{cranmer2020discovering}. 
We hope that our results will spur the improvements on the theoretical approximation theory of such trees and on the design of innovative techniques to train EML-type trees more effectively.

\clearpage
\bibliographystyle{plain}
\small{\bibliography{bibliography}}

\newpage
\appendix
\section{Notation}
\label{appendix:Notation}
We briefly recall the multi-index notation. Given a $d$-tuple of non-negative integers $\alpha \in \N_0^d$ with $d \in \N$, we have $|\alpha| = \alpha_1 + \cdots + \alpha_d$. Additionally, if $\alpha, \beta \in \N_0^d$ are both multi-indices, then $\alpha \leq \beta$ if and only if $\alpha_i \leq \beta_i$ for all $1 \leq i \leq d$. And, given a function $f : \Omega \to \R$ where $\Omega \subset \R^d$, we denote
\begin{equation}
    D^\alpha f = \fr{\partial^{|\alpha|} f}{\partial x_1^{\alpha_1} \cdots \partial x_d^{\alpha_d}}
\end{equation}
to be the classical or weak derivative of $f$.

Moreover, we use $\bigtimes$ to denote the Cartesian product of the provided sets.

\section{Sobolev Spaces}
\label{appendix:sobolevSpaces}
We recall the definitions of some basic functional spaces. 

The Lebesgue spaces, denoted by $L^p (\Omega)$, are defined as 
\begin{equation}
    L^p (\Omega) = \left\{ f: \Omega \to \R \ \middle| \ \OInt{|f(x)|^p} < \infty \right\},
\end{equation}
for $1 \leq p < \infty$, equipped with the norm
\begin{equation}
    \LnormD{f}{p}{\Omega} = \left( \OInt{|f(x)|^p }\right)^\fr{1}{p}.
\end{equation}
When $p = \infty$, the space $L^\infty (\Omega)$ consists of functions that are essentially bounded on $\Omega$ and whose norm is given by 
\begin{equation}
    \LnormD{f}{\infty}{\Omega} = \operatorname*{ess\,sup}_{x \in \Omega} \, |f(x)| = \inf \left\{ C \geq 0 \ \mid \ |f(x)| \leq C \text{ for almost every } x \in \Omega \right\}.
\end{equation}
Moreover, the Sobolev space $W^{k,p}$ is defined as 
\begin{equation}
    W^{k,p} (\Omega) = \left\{ f \in L^p (\Omega) \ \mid \  D^\alpha f \in L^p (\Omega) \text{ for all } \alpha \in \N_0^d \text{ with } |\alpha| \leq k \right\}.
\end{equation}
We also define associated seminorms. For $1 \le p < \infty$, we have
\begin{equation}
    |f|_{W^{m,p} (\Omega)} = \left( \sum_{|\alpha| =m} \LnormD{D^\alpha f}{p}{\Omega}^p \right)^\fr{1}{p}, \qquad \text{for } m =0,\dots,k,
\end{equation}
and for $p=\infty$, we have
\begin{equation}
    |f|_{W^{m,\infty} (\Omega)} = \max_{|\alpha| = m} \LnormD{D^\alpha f}{\infty}{\Omega} , \qquad \text{for } m =0,\dots,k.
\end{equation}
Based on these seminorms, we can define the following norms for $1 \le p < \infty$,
\begin{equation}
    \WnormD{f}{k}{p}{\Omega} = \left( \sum_{m=0}^k |f|_{W^{m,p} (\Omega)}^p \right)^\fr{1}{p},
\end{equation}
and for $p = \infty$,
\begin{equation}
    \WnormD{f}{k}{\infty}{\Omega} = \max_{0 \leq m \leq k} |f|_{W^{m,\infty} (\Omega)}.
\end{equation}

In particular, when $p = 2$, the Sobolev spaces $H^k (\Omega) = W^{k, 2} (\Omega)$ (with $k \in \N_0$) are Hilbert spaces with the corresponding norms $\HnormD{\cdot}{k}{\Omega} = \lVert \cdot \rVert_{W^{k,2} (\Omega)}$ and seminorms $| \cdot |_{H^m} = | \cdot |_{W^{m,2}}$ for any integer $m$ such that $0 \leq m \leq k$.

Additionally, the space of $k$-times continuously differentiable functions $C^k (\Omega)$ is endowed with  the norm $\lVert f \rVert_{C^k (\Omega)} = \lVert f \rVert_{W^{k, \infty} (\Omega)}$.

We will make frequent use of the following Banach algebra property.
\begin{lem}
\label{banachAlgebraProperty}
Let $d \in \N$, $k \in \N_0$, and $\Omega \subset \R^d$. If $f,g \in W^{k, \infty} (\Omega)$, then
\begin{equation}
    \WnormD{fg}{k}{\infty}{\Omega} \leq 2^k \WnormD{f}{k}{\infty}{\Omega} \WnormD{g}{k}{\infty}{\Omega}.
\end{equation}
\end{lem}

\section{Useful Definitions and Lemmas}
\label{appendix:definitionsAndLemma}
We recall the definition of a star-shaped set, which is a necessary geometric requirement for the approximation of functions by polynomials in Sobolev spaces (see Lemma \ref{BrambleHilbert}). 
\begin{defi}
\label{starShapedSet}
    A set $A \subset \R^d$ is star-shaped with respect to every point in a subset $B \subset A$ if
    \begin{equation}
        \forall y \in B, \forall x \in A, [y,x] \subset A,
    \end{equation}
    where $[y,x] = \{(1-t) y + tx \mid t \in [0,1] \}$ is the straight line in $\R^d$ connecting the point $y$ to the point $x$. 
\end{defi}
We recall the Bramble-Hilbert lemma for Sobolev spaces, as proved in \cite{de2021approximation}, with similar statements appearing in \cite{duran1983polynomial} and \cite{verfurth1999note}.
\begin{lem}
\label{BrambleHilbert}
    Let $\Omega \subset \R^d$ be an open and bounded set of diameter $0 < h < e^{-1/2} d^{-3/2}$ which is star-shaped with respect to every point in an open ball $B \subset \Omega$ with diameter $\rho h$. Then, for every $f \in W^{s, \infty} (\Omega)$, there exists a polynomial $\hat f$ of degree at most $s-1$ such that for any $k \in \N_0$ with $k < s$, it holds that
    \begin{equation}
        \WnormD{f - \hat f}{k}{\infty}{\Omega} \leq \fr{\pi^{1/4} \sqrt{s} (d \sqrt{de}h)^{s-k}}{(s-k-1)!} \, |f|_{W^{s, \infty}(\Omega)}.
    \end{equation}
\end{lem}

We recall the definition of the $\tanh$ function as
\begin{equation}
    \s(x) \coloneq \tanh (x) = \fr{e^{x} - e^{-x}}{e^x + e^{-x}}.
\end{equation}

We will frequently use the following property of $\tanh$ functions under differentiation.
\begin{lem}
\label{tanhProperty}
Given $m \in \N$, it holds that
\begin{equation}
    |\s^{(m)} (x)| \leq (2m)^{m+1} \min \{\exp (-2x), \exp(2x) \},
\end{equation}
for all $x \in \R$.
\end{lem}

\section{Required Building Blocks}
\label{appendix:basicBuildingBlocks}
In this section, we present the explicit constructions of the basic binary operations, exponentiation, and the hyperbolic tangent function, all of which are used in the approximation proof (Theorem \ref{main_theorem}).

\subsection{Basic Binary Operations}
In Figure \ref{fig:basicBinary}, we provide the EML trees (with their associated parameters) to produce the basic binary operations of addition, subtraction, mutiplication, and division. The addition and subtraction blocks are used in the constructions of the multiplication and division blocks respectively. In all EML diagrams, we denote EML atoms by rectangles labeled $\text{EML}_{\theta _{i}}$ (where $i$ is an integer) and blocks by squares indicating the corresponding operation or function. In the first few examples, we explicitly write (one possible) value of the $\theta_i$'s at the bottom of the figures. Additionally, the functions written above the arrows between EMLs represent the output of the EML from which the arrow originates.
\begin{figure}[htbp]
  \centering
  \begin{subfigure}{0.48\textwidth}
    \centering
    \includegraphics[width=\textwidth]{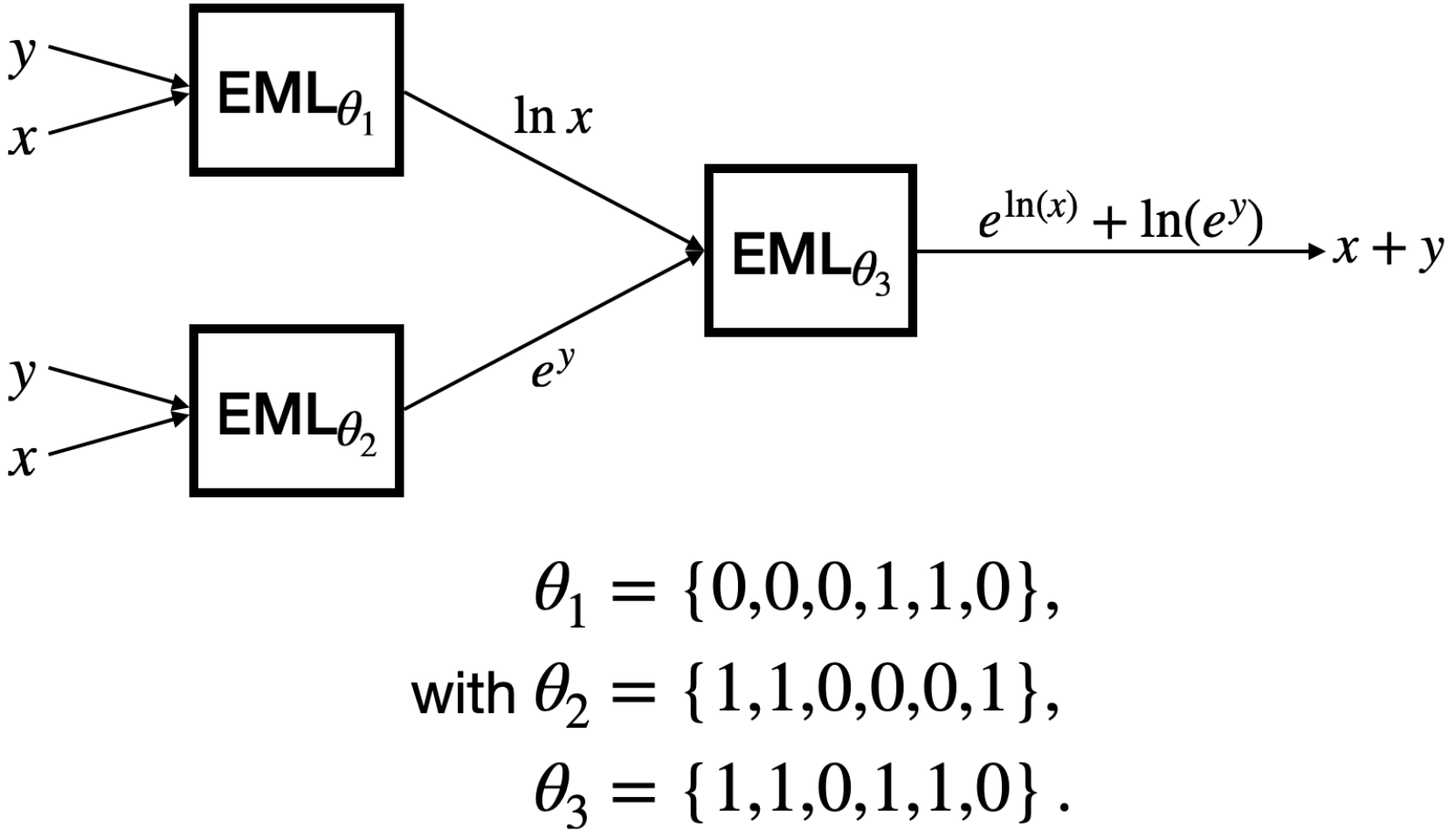}
    \caption{Addition (3)}
    \label{fig:addition}
  \end{subfigure}
  \hfill
  \begin{subfigure}{0.48\textwidth}
    \centering
    \includegraphics[width=\textwidth]{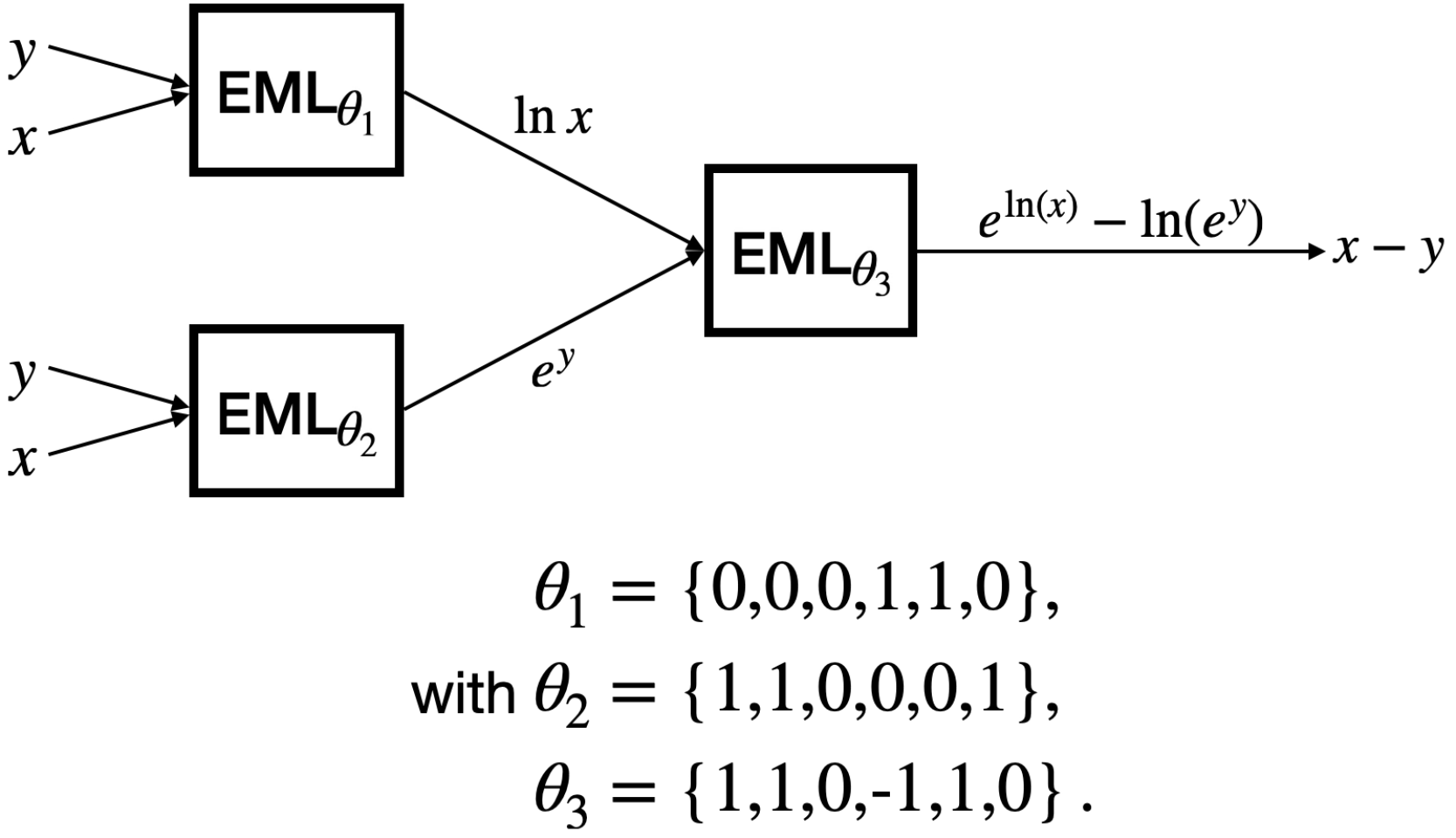}
    \caption{Subtraction (3)}
    \label{fig:subtraction}
  \end{subfigure}
  \vspace{1em}

  \begin{subfigure}{0.65\textwidth}
    \centering
    \includegraphics[width=\textwidth]{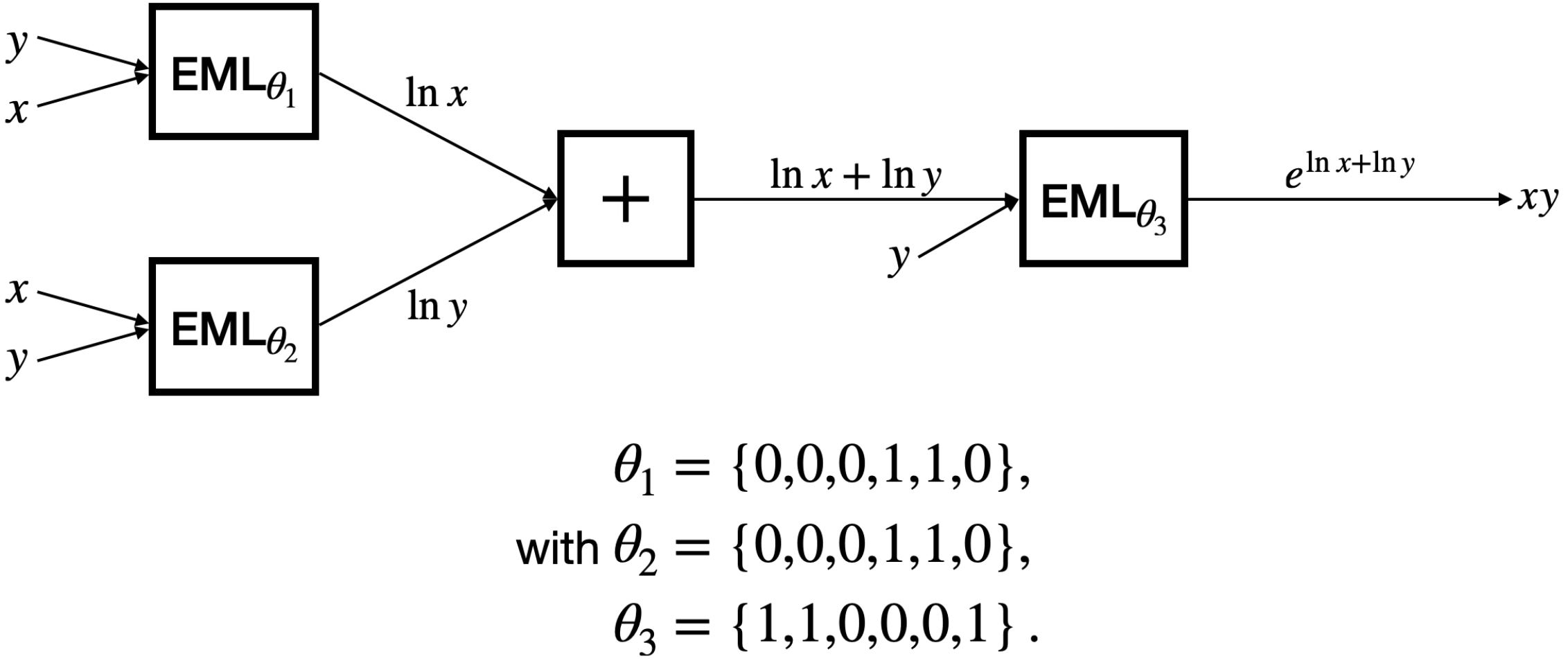}
    \caption{Multiplication (6)}
    \label{fig:multiplication}
  \end{subfigure}

  \vspace{1em}

  \begin{subfigure}{0.65\textwidth}
    \centering
    \includegraphics[width=\textwidth]{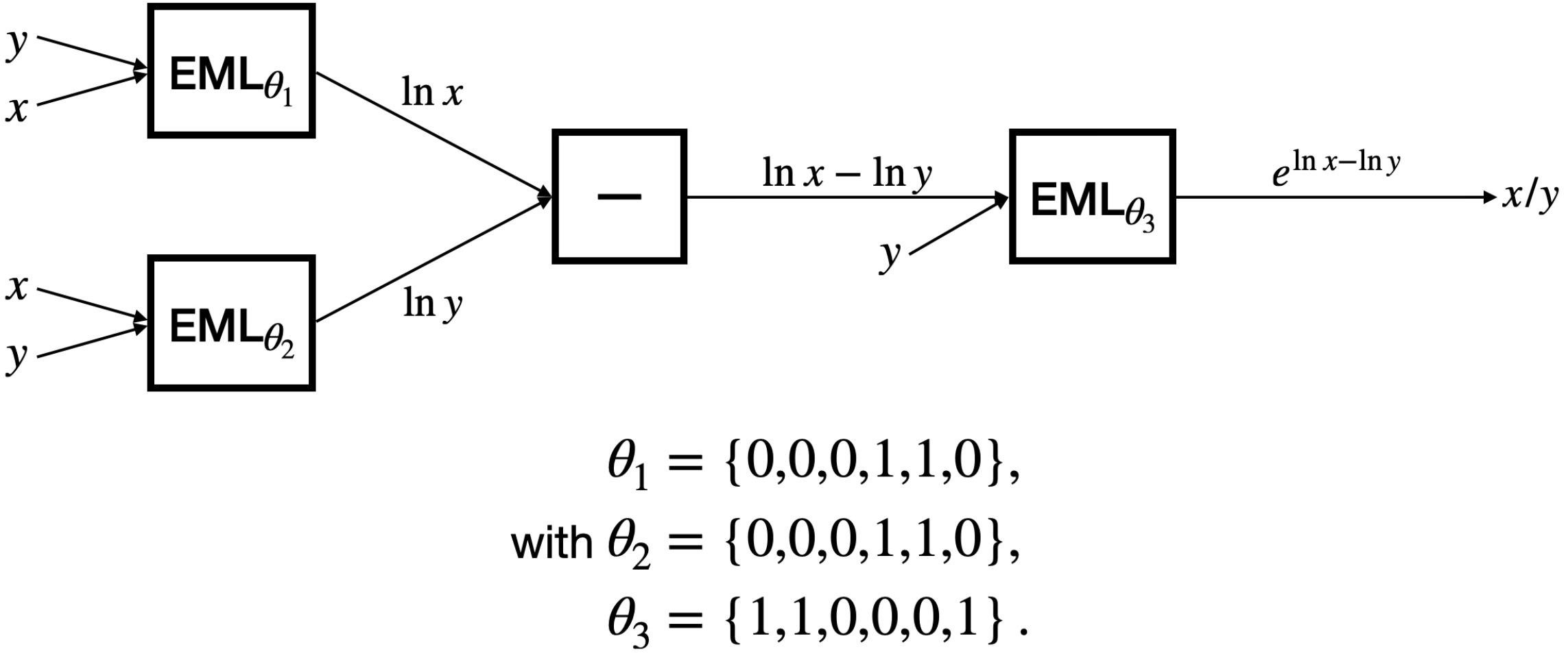}
    \caption{Division (6)}
    \label{fig:division}
  \end{subfigure}
  
  \caption{The Basic Binary Operations}
  \label{fig:basicBinary}
\end{figure}

As observed in Figure \ref{fig:basicBinary}, we note that the size and depth requirements of these binary operations are summarized in Table \ref{tab:binary_ops}. The inputs $x$ and $y$ are strictly positive.

\begin{table}[h]
\centering
\begin{tabular}{c|cc}
Operator & $N(\cdot)$ & $\mathrm{Depth}(\cdot)$ \\
\hline
$x +_\theta y$ & 3 & 2 \\
$x -_\theta y$ & 3 & 2 \\
$x \times_\theta y$ & 6 & 4 \\
$x \div_\theta y$ & 6 & 4 \\
$x^n_\theta, \ m x^n_\theta$ & 2 & 2 \\
$\tanh_\theta (x), \ m\tanh_\theta(x)$ & 8 & 5
\end{tabular}
\caption{Size and depth of the operations in Appendix \ref{appendix:basicBuildingBlocks}.}
\label{tab:binary_ops}
\end{table}

\subsection{Exponentiation}
In our proof, we will also make frequent use of exponentiation to a constant power. We provide explicit constructions for EML trees that perform $x^n$ and $mx^n$ for constants $m,n \in \R$ in Figure \ref{fig:exponentiating}.
\begin{figure}[htbp]
  \centering
  \begin{subfigure}{0.6\textwidth}
    \centering
    \includegraphics[width=\textwidth]{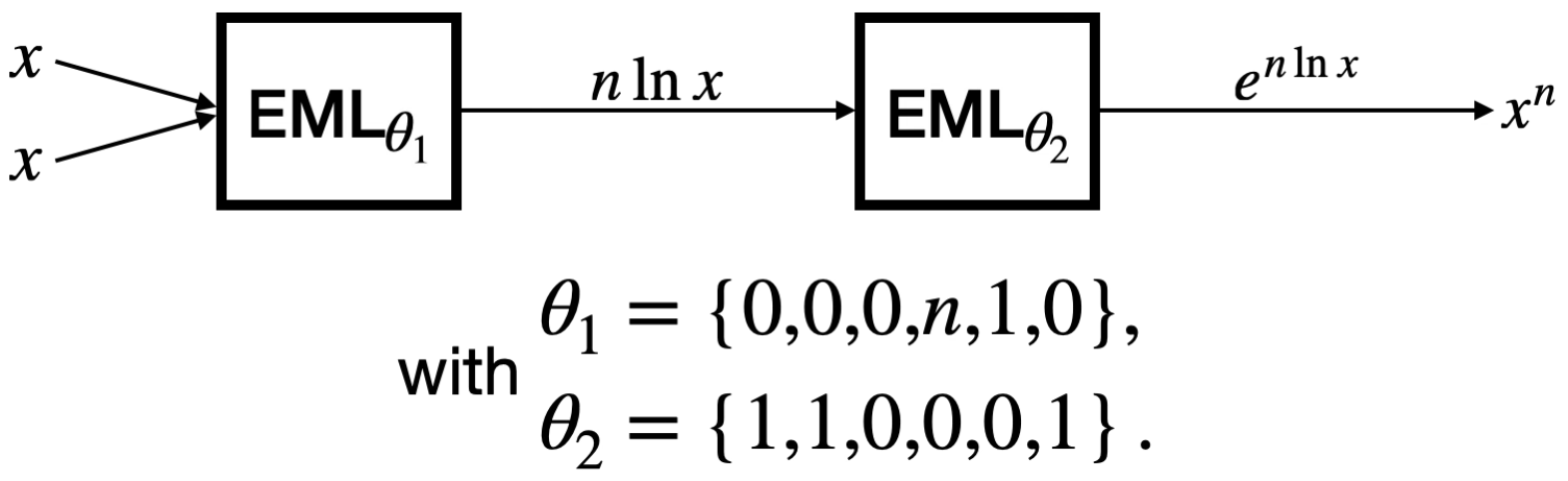}
    \caption{Exponentiation to a constant $n$ (2)}
    \label{fig:exponentiationSec}
  \end{subfigure}
  \vspace{1em}
  \begin{subfigure}{0.6\textwidth}
    \centering
    \includegraphics[width=\textwidth]{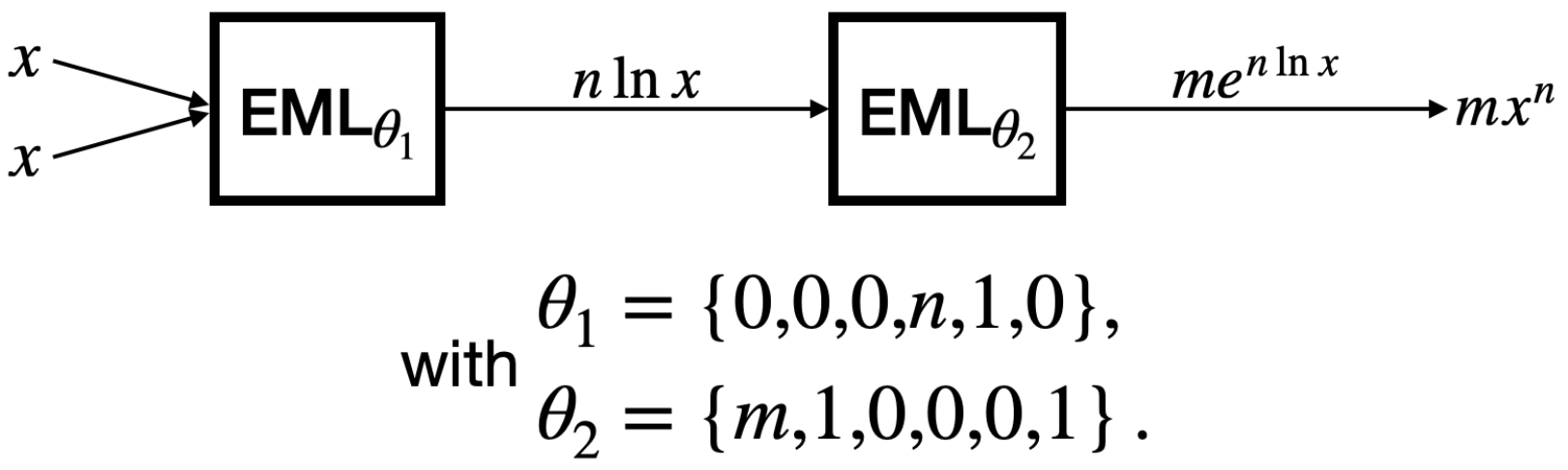}
    \caption{Exponentiation to a constant $n$ with prefactor $m$ (2)}
    \label{fig:exponentiationWithPrefactor}
  \end{subfigure}
  
  \caption{Exponentiating}
  \label{fig:exponentiating}
\end{figure}

Note that in Figure \ref{fig:exponentiationWithPrefactor}, we construct $m x^n$ directly, which we use to represent the monomial terms of polynomials (in Appendix \ref{appendix:constructingPolynomials}). This is because it is more efficient to incorporate the coefficient within the exponentiation block itself, rather than introducing a separate multiplication block to combine a constant with the $x^n$ block.

For both exponentiation blocks (\ref{fig:exponentiationSec} and \ref{fig:exponentiationWithPrefactor}), the size and depth of the associated EML trees are $2$ (and appended to Table \ref{tab:binary_ops}).

\subsection{Hyperbolic Tangent}
The hyperbolic tangent is constructed, using the exponential definition, by first getting $e^{2x} - 1$ and $e^{2x} + 1$ (via two separate EML blocks) and then diving them, as in Figure \ref{fig:tanhFigure}. The size of the $\tanh$ block is 8 and its depth is $5$, which are appended to Table \ref{tab:binary_ops}. 

We can similarly obtain $m \tanh x$, which is useful for the construction of the approximate partitions of unity, by amending the division block to include the prefactor $m$. This is done by changing $\theta_3$ in the division block (Figure \ref{fig:division}) to be $\theta_3 = \{m, 1, 0, 0, 0 , 1\}$.
\begin{figure}[htbp]
    \centering
    \includegraphics[width=0.6\linewidth]{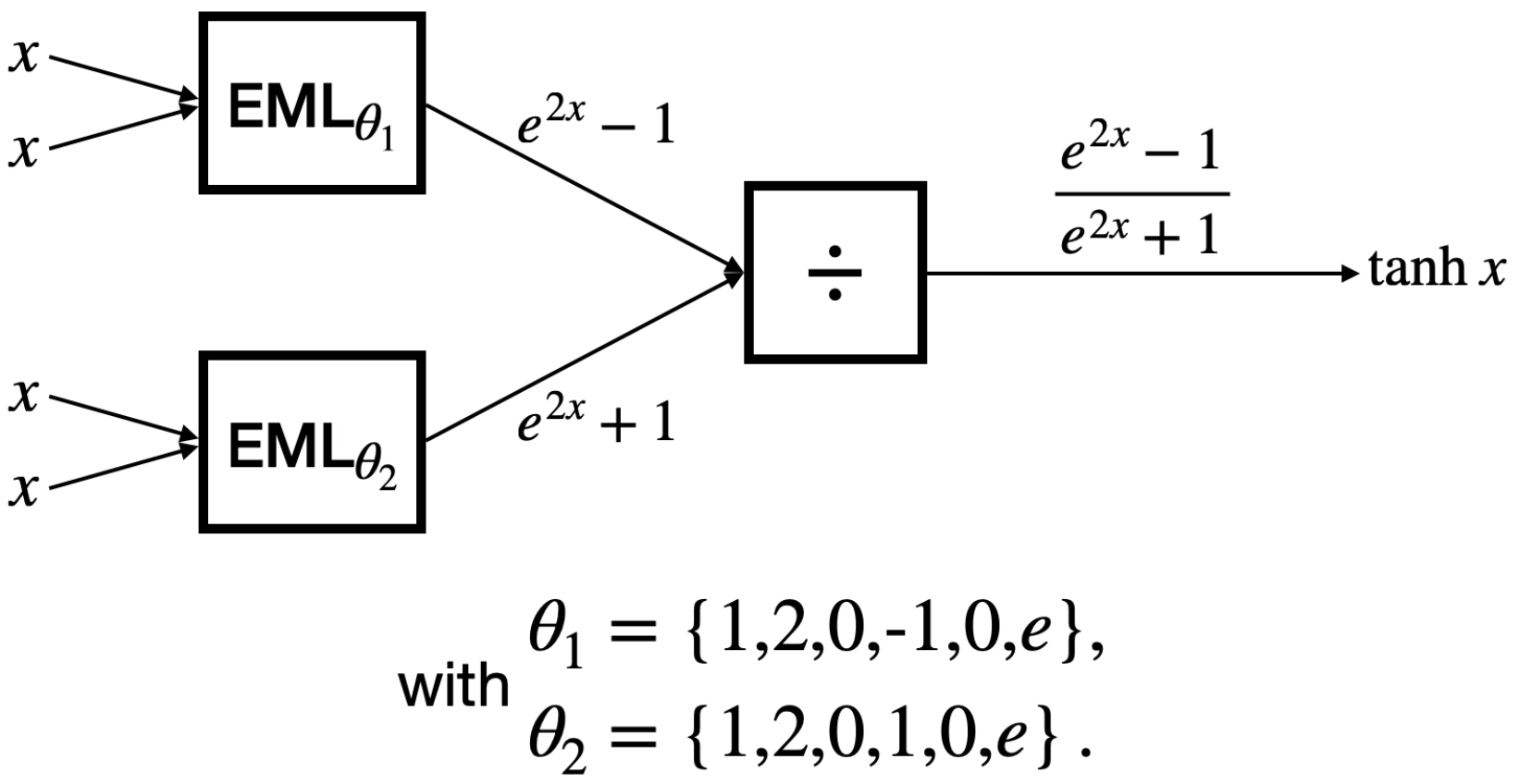}
    \caption{Hyperbolic Tangent (8)}
    \label{fig:tanhFigure}
\end{figure}

\section{Constructing Polynomials}
\label{appendix:constructingPolynomials}
The following two subsections detail how to construct univariate and multivariate polynomials using EML trees and provide upper bounds for the number of required EML blocks in each case.

\subsection{Univariate Polynomials}
\label{sec_univariate_polynomial}
A univariate polynomial of degree $k$, denoted by
\begin{equation}
    p(x) = \sum_{i=0}^k a_i x^i = a_0 + a_1 x + a_2 x^2 + \cdots + a_k x^k,
\end{equation}
with $x > 0$ and with positive coefficients $a_0, a_1, \dots, a_k > 0$,
can be (naively) constructed using EML blocks as in Figure \ref{fig:univariatePolynomial}, wherein we used two simple constituent blocks: addition (Figure \ref{fig:addition}) and exponentiation to a constant with a prefactor (Figure \ref{fig:exponentiationWithPrefactor}).

We impose that $x > 0$ because our universal approximation theorem (Theorem \ref{main_theorem}) is on the the interval $(0,1]^d$, in order to circumvent the undefinedness of the natural logarithm for nonpositive inputs.

The need for the positivity condition on the coefficients is because the exponentiation blocks are fed into an addition block; if a coefficient is negative, then the output of the exponentiation block is too (since $x > 0$), and so it creates a problem if it is fed into the logarithm of the addition block. The case with negative coefficients will be treated slightly differently below.

\begin{figure}[htbp]
    \centering
    \includegraphics[width=0.9\linewidth]{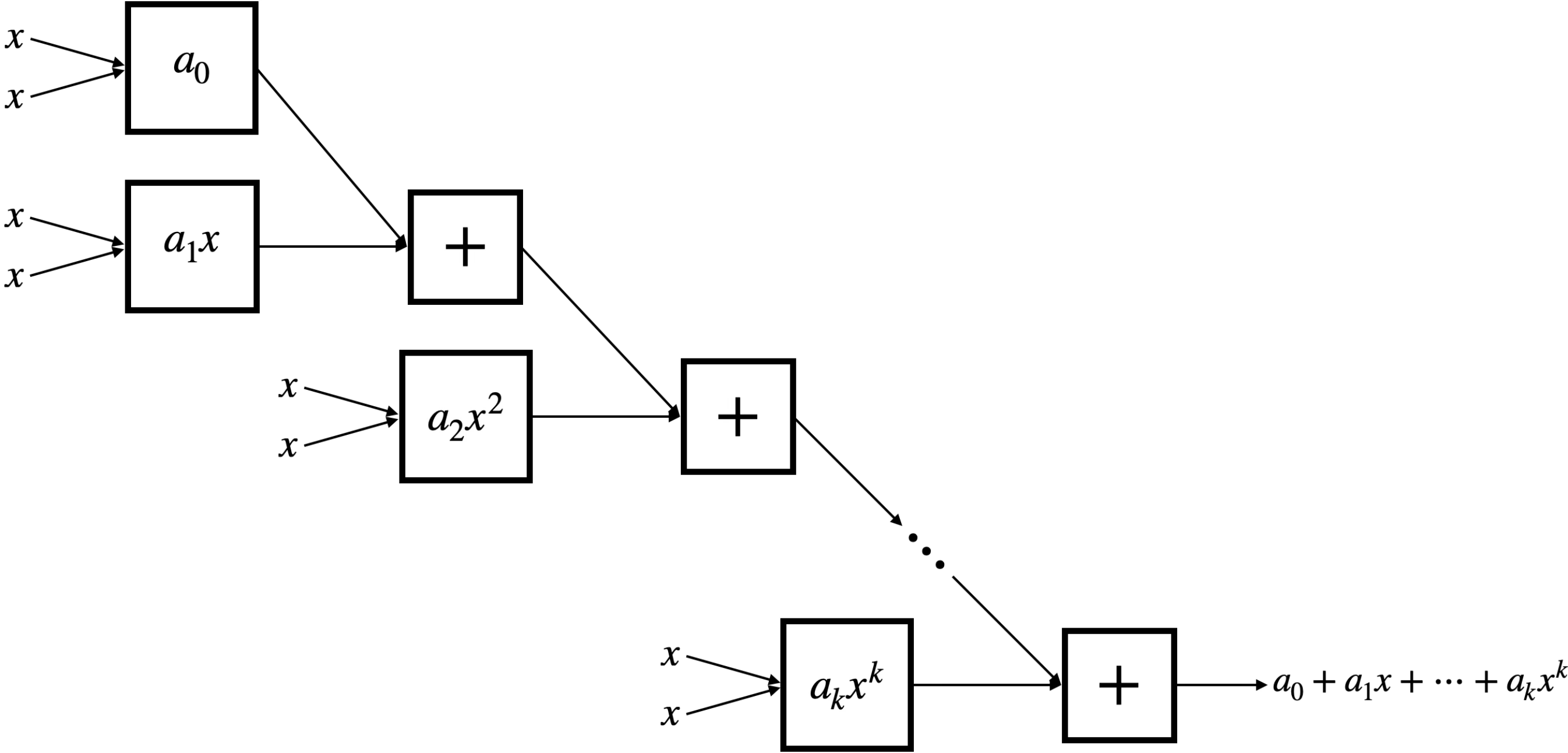}
    \caption{Univariate Polynomial ($5s-3$)}
    \label{fig:univariatePolynomial}
\end{figure}

The upper bound on the number of required EML blocks is obtained by noting that we used
\begin{itemize}
    \item $k$ addition blocks (each requiring $3$ EML atoms), and
    \item $k+1$ exponentiation with a prefactor blocks (each requiring $2$ EML atoms).
\end{itemize}

So, constructing an EML tree $p_\theta$ such that $p_\theta = p$ (as per the above procedure) requires $N(T_\theta) = 3k + 2(k+1) = 5k+2$ and $\text{Depth} (T_\theta) = 2k + 2$, since the longest path starts with the $x$ feeding into the $a_0$ block and passes through all {$k$} addition blocks.

Hence, we phrase the above conclusion as the lemma:
\begin{lem}[Construction of Univariate Polynomials with Positive Coefficients]
\label{lem_univariate_polynomial}
Given a univariate polynomial $p$ of degree $k$ {defined over $x > 0$ and with positive coefficients}, there exists an EML tree $p_\theta$ (following the explicit construction in Figure \ref{fig:univariatePolynomial}), such that $p_\theta = p$ with {$N(p_\theta) \leq 5k + 2$ and $\text{Depth}(p_\theta) \leq 2k+2$}.
\end{lem}

{In the case where the polynomial contains negative coefficients, we perform the following decomposition (again with $x > 0$),
\begin{equation}
\label{polynomial_decomposition}
    p (x) = p_+ (x) - p_- (x) \coloneqq \sum_{\substack{0 \leq i \leq k \\ a_i > 0}} a_i x^i - \sum_{\substack{0 \leq i \leq k \\ a_i < 0}} | a_i |  x^i.
\end{equation}
Correspondingly, all the intermediary operations to construct $p_+$ and $p_-$ contain only positive numbers, that is they can be constructed via Lemma \ref{lem_univariate_polynomial}, and they output positive numbers, which can then be joined via the subtraction block to give $p$. Note that if the final aim of the EML tree is to represent $p$, then it doesn't cause any issue if $p(x)$ outputs a negative number.}

{Therefore, the above decomposition leads to the following lemma:
\begin{lem}[Construction of Univariate Polynomial]
\label{lem_univariate_polynomials_with_decomposition}
Given a univariate polynomial $p$ of degree $k$ (with both positive and negative coefficients) defined for $x > 0$, there exists an EML tree $p_\theta$, such that $p_\theta = p$ with $N(p_\theta) \leq 10k + 7$ and $\text{Depth} (p_\theta) \leq 2k + 4$.
\end{lem}}

\subsection{Multivariate Polynomials}
We recall that, for $\mathbf{x} = (x_1, \dots, x_d) \in \R^d$ {with $x_i > 0$ and $1 \leq i \leq d$}, a multivariate polynomial $p$ of total degree at most $k$ is of the form
\begin{equation}
\label{multivariatePolynomial}
    p(\mathbf{x}) = \sum_{|\alpha| \leq k} c_\alpha x_1^{\alpha_1} x_2^{\alpha_2} \cdots x_d^{\alpha_d},
\end{equation}
where $\alpha = (\alpha_1, \alpha_2, \dots, \alpha_d)$ is a multi-index with $|\alpha| = \alpha_1 + \alpha_2 + \cdots + \alpha_d$, and $c_\alpha$ are the associated coefficients. Such a polynomial (of degree $\leq k$) has $\binom{d+k}{k}$ monomial terms.

{As in Section \ref{sec_univariate_polynomial}, we first aim to construct an EML tree for the polynomial \eqref{multivariatePolynomial} with positive coefficients $c_\alpha$. To do that,} we (1) construct the monomials $c_\alpha x_1^{\alpha_1}$, $x_2^{\alpha_2}$, $\dots$, $x_d^{\alpha_d}$, (2) multiply these monomials together (separately for each $\alpha$), (3) do steps (1) and (2) for every monomial for each $\alpha$ with $|\alpha| \leq k$, and (4) add them together. A breakdown of the required EML tree sizes and depths for each step is as follows:
\begin{enumerate}[label=(\arabic*)]
    \item For a certain $\alpha$, each monomial, of which there are $d$, requires an exponentiation block (and the first has a prefactor), totaling $2d$ EML atoms;
    \item Again for each $\alpha$, we need to multiply the $d$ monomials, requiring $d-1$ multiplication blocks and, thus, $6(d-1)$ EML atoms; thus, each term, considering both getting monomials and multiplying them requires $2d + 6(d-1) = 8d - 6$ EML atoms; 
    \item Since we have $\binom{d+k}{k}$ terms in the polynomial, we do steps (1) and (2) $\binom{d+k}{k}$ times, thus requiring $(8d-6)\binom{d+k}{k}$ EML atoms;
    \item We finally sum over all terms, requiring $\binom{d+k}{k} - 1$ addition blocks, each consisting of $3$ EML atoms.
\end{enumerate}
Thus, in total, constructing the polynomial $p$ in \eqref{multivariatePolynomial} takes
\begin{equation}
    (8d-6) \binom{d+k}{k} + 3 \left( \binom{d+k}{k} - 1 \right) = (8d-3)\binom{d+k}{k} - 3
\end{equation}
EML atoms, and whose depth can be calculated (by drawing a schematic like Figure \ref{fig:univariatePolynomial}) to be
\begin{equation}
    2+4(d-1) + 2 \left( \binom{d+k}{k} - 1 \right) = 4d + 2 \binom{d+k}{k} -4,
\end{equation}
since the $x$'s feed into the first $c_\alpha x^{\alpha_1}$ block (of depth $2$), pass through $d-1$ multiplication (each of depth $4$), and then through $\binom{d+k}{k} - 1$ addition blocks (each of depth $2$).

The above justifies the following lemma:
\begin{lem}[Construction of Multivariate Polynomials with Positive Coefficients]
\label{lemma_multivariatePolynomial}
Given a $d$-dimensional multivariate polynomial $p$ defined for positive input, with degree $k$, and with positive coefficients, there exists an EML tree $p_\theta$, such that $p_\theta = p$ with
\begin{equation}
    N(p_\theta) \leq (8d-3) \binom{d+k}{k} - 3, \quad \text{ and } \quad \text{Depth}(p_\theta) \leq 4d + 2 \binom{d+k}{k} - 4.
\end{equation}
\end{lem}

{For a multivariate polynomial with both positive and negative coefficients, we perform the same decomposition as in \ref{polynomial_decomposition} and can thus conclude the following lemma:
\begin{lem}[Construction of Multivariate Polynomials] 
\label{lemma_multivariatePolynomial_with_decomposition}
    Given a $d$-dimensional multivariate polynomial $p$ defined for positive input and with degree $k$, there exists an EML tree $p_\theta$, such that $p_\theta = p$ with
    \begin{equation}
        N(p_\theta) \leq 2(8d - 3) \binom{d+k}{k} - 3, \text{ and } \text{Depth}(p_\theta) \leq 4d + 2\binom{d+k}{k} - 2.
    \end{equation}
\end{lem}}

\section{Approximating the partition of unity}
\label{appendix:partitionsOfUnity}
In this section, we aim to construct an approximate partition of unity, similar to Section 4 of \cite{de2021approximation}. We do this because indicator functions are unsuitable to be represented by EML trees due to the fact that they contain a discontinuity and cannot be suitably represented by chaining elementary functions constructed from compositions of EML. 
For this reason, we construct smooth approximations of indicator functions on subsets of $[0,1]^d$ via $\tanh$ functions. When suitably shifted and parametrized, these functions output values close to $1$ within the subset and close to $0$ outside of it.

\subsection{Properties of the approximate partition of unity}
% To that effect, for every $j \in \mathcal{A}^N$, 
We now recall some of the properties of the partition of unity of \cite{de2021approximation}.

For a point $y \in \R$ and a parameter $\alpha \in \R$ (to be chosen later), we define the one-dimensional approximate indicator functions as
\be\begin{aligned}
\label{partitions_numero_uno}
    \rho_1^N(y) &= \fr{1}{2} - \fr{1}{2} \s \left( \alpha \left( y - \fr{1}{N} \right) \right), \\
    \rho_m^N(y) &= \fr{1}{2} \s \left( \alpha \left( y - \fr{m-1}{N} \right) \right)  - \fr{1}{2} \s \left( \alpha \left(y - \fr{m}{N} \right) \right) \quad \text{for } 2 \leq m \leq N-1, \\
    \rho_N^N(y) &= \fr{1}{2} \s \left( \alpha \left( y - \fr{N-1}{N} \right) \right) + \fr{1}{2}.
\end{aligned}\ee
We can clearly see that if $\alpha$ is large enough, these act as soft indicator functions over the intervals $\left(\fr{m-1}{N}, \fr{m}{N} \right)$ (depending on the subscript of $\rho$). 

To choose $\alpha$, we note that since $\s (x) \to 1$ as $x \to \infty$ and its derivatives also saturate ($\s^{(m)}(x) \to 0$ with $x \to \infty$), we can choose $R > 0$ large enough such that $|\s^{(m)}|$ is decreasing over $[R,\infty)$ for all $1 \leq m \leq k$.

So, given $\varepsilon > 0$, we choose $\alpha$, depending on $N$ and $\varepsilon$, large enough such that
\begin{equation}
\label{defOfConstantR}
    \fr{\alpha}{N} \geq R, \qquad 1 - \s \left( \fr{\alpha}{N} \right) \leq \varepsilon, \qquad \alpha^m \left| \s^{(m)} \left( \fr{\alpha}{N} \right) \right| \leq \varepsilon \text{ for all } 1 \leq m \leq k,
\end{equation}
justified also by Lemma A.4 in \cite{de2021approximation}. Consequently, $\alpha$ can be chosen to satisfy the above conditions and has the explicit value
\begin{equation}
\label{valueOfAlpha}
    \alpha = N \max \left\{ R, \ln \left( \fr{(2k)^{k+1} (Nk)^k}{e^k \varepsilon} \right) \right\},
\end{equation}
(see Lemma A.5 in \cite{de2021approximation} for justification).

Importantly, for $j = (j_i)_{i=1}^d \in \mathcal{A}^N$ (where $\mathcal{A}^N$ is the index set \eqref{indexan}), we define
\begin{equation}
\label{partition_of_unity_def}
    \Phi_j^{N} (x) = \prod_{i=1}^d \rho_{j_i}^{N} (x_i),
\end{equation}
which acts as our approximate partition of unity. We also define the set $\mathcal{V}_d$ as
\begin{equation}
    \mathcal{V}_d = \left\{ v \in \mathbb{Z}^d \ \middle| \ \max_{1 \leq i \leq d} |v_i| \leq 1 \right\},
\end{equation}
which allows us to choose only the cubes directly surrounding a cube of interest.

Lemmas \ref{lemma1_partitionsOfUnity} and \ref{lemma2_partitionsOfUnity} below prove that the set of functions $\{ \Phi_j^N \}_{j \in \mathcal{A}^N}$  acts as an approximate partition of unity in the sense that they satisfy the following two properties: for every $j \in \mathcal{A}^N$,
\begin{equation}
    \sum_{v \in \mathcal{V}_d} \Phi_{j+v}^{N, d} \approx 1 \qquad \text{and} \qquad \sum_{\substack{v \notin \mathcal{V}_d, \\ j+v \in \mathcal{A}^N}} \Phi_{j+v}^{N,d} \approx 0.
\end{equation}

We recall the below two lemmas from \cite{de2021approximation}. The first provides a quantitative bound for how close to 1 is the contribution of $\Phi_j^N$ and its direct neighbors on $I_j^N$ (constructed in \eqref{definitionOfCubes}).

\begin{lem}[Lemma 4.1 of \cite{de2021approximation}]
\label{lemma1_partitionsOfUnity}
If $0 < \varepsilon < 1/4$, then
\begin{equation}
    \WnormD{\sum_{v \in \mathcal{V}_d} \Phi_{j+v}^{N} - 1}{k}{\infty}{I_j^N} \leq 2^{dk} d \varepsilon.
\end{equation}
\end{lem}

The second lemma guarantees that the contribution of the functions far away from the cube $I_j^N$ is small.
\begin{lem}[Lemma 4.2 of \cite{de2021approximation}]
\label{lemma2_partitionsOfUnity}
Let $k \in \N_0$ and $v \in \mathbb{Z}^d$ with $v_i \geq 2$ for all $1 \leq i \leq N$. Then, it holds that
\begin{equation}
\WnormD{\Phi_{j+v}^{N}}{k}{\infty}{I_j^N} \leq \max \{1, (2k)^{2k} \alpha^k \} \, \varepsilon.
\end{equation}
\end{lem}

\subsection{Constructing the approximate partition of unity using EML trees}
{Next, we construct the $\rho_m^N$ functions for $2 \leq m \leq N-1$, according to the definition in \eqref{partitions_numero_uno}. All of the $\rho_m^N$ output values are larger than $0$, so no issue arises from feeding their output into subsequent EML functions. However, an issue can arise in the intermediary steps of the construction, wherein a negative number could result from the $\tanh$ functions. To circumvent this issue, we perform the following rearrangement of terms
\be\begin{aligned}
\label{decomposition_for_partitions_of_unity}
    \rho_m^N (y) &= \fr{1}{2} \s \left( \alpha \left( y - \fr{m-1}{N} \right) \right)  - \fr{1}{2} \s \left( \alpha \left(y - \fr{m}{N} \right) \right)\\
    &= \fr{1}{2} \underbrace{\left( \fr{e^{\alpha \left( y - \fr{m-1}{N} \right)}}{e^{\alpha \left( y - \fr{m-1}{N} \right)} + e^{-\alpha \left( y - \fr{m-1}{N} \right)}}
        - \fr{e^{\alpha \left( y - \fr{m}{N} \right)}}{e^{\alpha \left( y - \fr{m}{N} \right)} + e^{-\alpha \left( y - \fr{m}{N} \right)}} \right)}_\text{positive} \\
    &\qquad + \fr{1}{2} \underbrace{\left( \fr{e^{-\alpha \left( y - \fr{m}{N} \right)}}{e^{\alpha \left( y - \fr{m}{N} \right)} + e^{-\alpha \left( y - \fr{m}{N} \right)}} 
        - \fr{e^{-\alpha \left( y - \fr{m-1}{N} \right)}}{e^{\alpha \left( y - \fr{m-1}{N} \right)} + e^{-\alpha \left( y - \fr{m-1}{N} \right)}}. \right)}_{\text{positive}}
\end{aligned}\ee
As noted in the equation above, the first and second terms are positive, and thus we do not encounter any sign issues with this decomposition.}

\begin{figure}[htbp]
    \centering
    \includegraphics[width=0.99\linewidth]{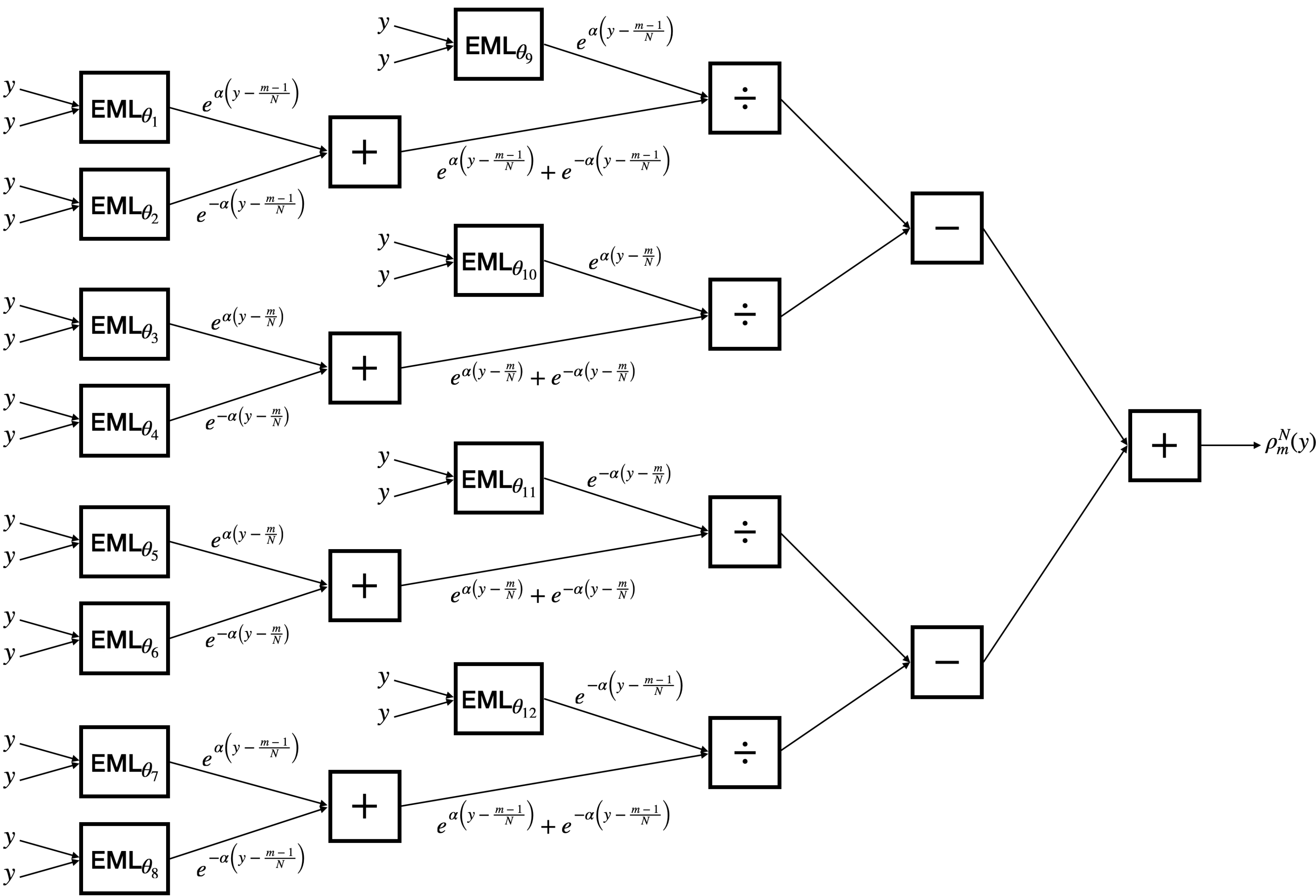}
    \caption{Construction of $\rho_m^N$ for $2\leq m \leq N-1$}
    \label{fig:partitionOfUnity}
\end{figure}

{This construction, represented in Figure \ref{fig:partitionOfUnity}, requires $12$ EML blocks, $5$ addition blocks (the last of which is chosen to have a $1/2$ prefactor), $4$ division block, and $2$ subtraction blocks. This boils down to a size of $12 + 5 \times 3 + 4 \times 6 + 2 \times 3 = 57$ and a depth of $1 + 2 + 4 + 2 + 2 = 11$.} 

{Similarly, to construct $\rho_1^N$, we rearrange the function so that we subtract positive numbers as before}
{\be\begin{aligned}
    \rho_1^N (y) &= \fr{1}{2} - \fr{1}{2} \s \left( \alpha \left( y - \fr{1}{N} \right) \right) \\
    &= \fr{1}{2} \underbrace{\left( 1 + \fr{e^{-\alpha \left( 1 - \fr{1}{N} \right)}}{e^{\alpha \left( 1 - \fr{1}{N} \right)} + e^{-\alpha \left( 1 - \fr{1}{N} \right)}} \right)}_\text{positive} - \fr{1}{2} \underbrace{\left( \fr{e^{\alpha \left( 1 - \fr{1}{N} \right)}}{e^{\alpha \left( 1 - \fr{1}{N} \right)} + e^{-\alpha \left( 1 - \fr{1}{N} \right)}} \right)}_\text{positive},
\end{aligned}\ee
and for $\rho_N^N$, we rewrite it as
\be\begin{aligned}
    \rho_N^N (y) &= \fr{1}{2} \s \left( \alpha \left( y - \fr{N-1}{N} \right) \right) + \fr{1}{2} \\
    &= \fr{1}{2} \underbrace{\left( 1 + \fr{e^{\alpha \left( y - \fr{N-1}{N} \right)}}{e^{\alpha \left( y - \fr{N-1}{N} \right)} + e^{-\alpha \left( y - \fr{N-1}{N} \right)}} \right)}_\text{positive} - \fr{1}{2} \underbrace{\fr{e^{-\alpha \left( y - \fr{N-1}{N} \right)}}{e^{\alpha \left( y - \fr{N-1}{N} \right)} + e^{-\alpha \left( y - \fr{N-1}{N} \right)}}}_\text{positive}.
\end{aligned}\ee
The size and depth of the EML blocks representing $\rho_1^N$ and $\rho_N^N$ can be upper bounded, for simplicity, by those for $\rho_m^N$ for $2 \leq m \leq N-1$.}

Then, we construct $\Phi_j^{N}$, using \eqref{partition_of_unity_def} by multiplying $d$ instances of the one-dimensional partitions $\rho_{j_i}^N$, {which doesn't cause any issues due to the positivity of the $\rho_m^N$ functions. This requires $d-1$ multiplication blocks, and thus the EML tree corresponding to $\Phi_j^{N}$ has size $57 d + 6 (d-1) = 63d - 6$ and depth $11 + 4(d-1) = 4d + 7$, yielding the following lemma:
\begin{lem}[Construction of Approximate Partitions of Unity]
\label{lemma_partitionsOfUnityEML}
    For $N,d \in \N$ and $j \in \mathcal{A}^N$, there exists an EML tree $\Phi_{j, \theta}^N$ such that $\Phi_{j, \theta}^N = \Phi_j^{N}$ with
    \begin{equation}
        N(\Phi_{j, \theta}^N) \leq 63d - 6, \quad \text{ and } \quad \text{Depth} (\Phi_{j, \theta}^N) \leq 4d + 7.
    \end{equation}
\end{lem}}

\section{Proof of Theorem \ref{main_theorem}}
\label{appendix:proofOfMainTheorem}
\begin{proof}
We set out to perform the steps for the proof as described in the main text above.

{\bf{Step 1. Approximation by Local Polynomials.}}
For a given $N \in \N$ and $j \in \mathcal{A^N}$, we recall the definitions of $I_j^N$ and $J_j^N$ from \eqref{definitionOfCubes}.
We calculate that the diameter of $J_j^N$ is $\fr{3 \sqrt{d}}{N}$. We notice that there exists a ball of diameter $3/N$ such that $J_j^N$ is star-shaped with respect to every point in this ball (per Definition \ref{starShapedSet}). Thus, for every $j \in \mathcal{A}^N$, by the Bramble-Hilbert lemma (Lemma \ref{BrambleHilbert}), there exists a polynomial $p_j^N$ of degree at most $s-1$ such that, for all $0 \leq k \leq s-1$
\be\begin{aligned}
    \WnormD{f - p_j^N}{k}{\infty}{J_j^N} &\leq \fr{\pi^{1/4} \sqrt{s}}{(s-k-1)!} \left(\fr{5 d^2}{N} \right)^{s-k} \, |f|_{W^{s,\infty} ([0,1]^d)} \\
    &\leq \max_{0 \leq m \leq k} \fr{\pi^{1/4} \sqrt{s} (5d^2)^{s-m}}{(s-m-1)!} \fr{|f|_{W^{s, \infty} ([0,1]^d)}}{N^{s-k}} \coloneqq \fr{\mathcal{C} (d, k, s, f)}{N^{s-k}},
\end{aligned}\ee
choosing $N > 5d^2$ (to obtain decay) and using $3 \sqrt{e} \leq 5$.

\begin{rem}
\label{rem:overlappingCubes}
    The reason why the polynomials are chosen to be good approximations on the overlapping cubes (not on the disjoint ones) becomes evident in the estimates \eqref{fTildebyPolynomialLInfty} and \eqref{fTildebyPolynomialInSobolev}, allowing us to control the terms $\LnormD{f - p_{i+v}^N}{\infty}{I_i^N}$ and $\WnormD{f - p_{i+v}^N}{k}{\infty}{I_i^N}$ with $i \in \mathcal{A}^N$ and $v \in \mathcal{V}_d$, using the fact that $I_i^N \subset J_{i+v}^N$ and bounding the norms. 
\end{rem}

{\bf{Step 2. Definition of the Global Polynomial.}} We then define the global piecewise polynomial over the entire domain $[0,1]^d$ by
\begin{equation}
    p^N = \sum_{j \in \mathcal{A}^N} p_j^N \chi_j,
\end{equation}
where $\chi_j$ denotes the indicator function on $I_j^N$. Since $p^N$ is composed using indicator functions (which are discontinuous), it cannot be adequately controlled in $W^{k, \infty}$ spaces where $k \geq 1$. This is why we introduce
\begin{equation}
    \tilde{p}^N = \sum_{j \in \mathcal{A}^N} p_j^N \Phi_j^{N},
\end{equation}
as the global polynomial using the approximate partitions of unity and which we will ultimately represent using an EML tree. Similarly, we define $\tilde{f}^N$ as
\begin{equation}
    \tilde{f}^N = \sum_{j \in \mathcal{A}^N} f \, \Phi_j^{N}.
\end{equation}

{\bf{Step 3. Approximation of $f$ by $\tilde{f}^N$.}} We first prove that $\tilde{f}^N$ is close to $f$. Considering an arbitrary $i \in \mathcal{A}^N$ and any $k \geq 1$, we employ the property \ref{banachAlgebraProperty}, Lemmas \ref{lemma1_partitionsOfUnity} and \ref{lemma2_partitionsOfUnity}, and the explicit value of $\alpha$ (from \eqref{valueOfAlpha}) to get
\be\begin{aligned}
    &\WnormD{f - \tilde{f}^N}{k}{\infty}{I_i^N} = \WnormD{f - \sum_{j \in \mathcal{A}^N} f \, \Phi_j^{N}}{k}{\infty}{I_i^N} \\
    &\quad \leq 2^k \WnormD{f}{k}{\infty}{I_i^N} \WnormD{1 - \sum_{v \in \mathcal{V}_d} \Phi_{i+v}^{N}}{k}{\infty}{I_i^N} + 2^k \WnormD{f}{k}{\infty}{I_i^N} \WnormD{\sum_{\substack{j \in \mathcal{A^N} \\ j-i \notin \mathcal{V}_d}} \Phi_j^{N}}{k}{\infty}{I_i^N} \\
    &\quad \leq 2^k \WnormD{f}{k}{\infty}{I_i^N} \left(2^{kd}d \varepsilon + N^d (2k)^{2k} \alpha^k \varepsilon \right) \\
    &\quad \leq 2^k \WnormD{f}{k}{\infty}{I_i^N} \left(2^{kd}d + N^d (2k)^{2k} N^k (k+1)^k \max \left\{R^k, \ln^k \left( \fr{2Nk^2}{\varepsilon^\fr{1}{k+1} e} \right) \right\} \right) \varepsilon \\
    &\quad \leq \delta (2(k+1))^{3k} \max \left\{R^k, \ln^k \left( \fr{2Nk^2}{\varepsilon^\fr{1}{k+1} e} \right) \right\} \fr{\mathcal{C}(d,k,s,f)}{N^{s-k}},
\end{aligned}\ee
with $\varepsilon$ chosen to satisfy
\begin{equation}
    \varepsilon \leq \fr{\delta \mathcal{C} (d,k,s,f)}{2^{(k+1) d} d N^{s+d} \, \WnormO{f}{k}{\infty}}.
\end{equation}
Similarly, for $k=0$, we have
\be\begin{aligned}
    \LnormD{f - \tilde{f}^N}{\infty}{I_i^N} \leq \LnormD{f}{\infty}{I_i^N} (d\varepsilon + N^d \varepsilon) \leq \fr{\delta}{3} \fr{\mathcal{C}(d,k,s,f)}{N^{s-k}}.
\end{aligned}\ee

{\bf{Step 4. Approximation of $\tilde{f}$ by the Global Polynomial $\tilde{p}^N$.}}
We note that for $k = 0$, we employ properties of the $L^p$ norm and Lemmas \ref{lemma1_partitionsOfUnity} and \ref{lemma2_partitionsOfUnity} to obtain
\be\begin{aligned}
\label{fTildebyPolynomialLInfty}
    &\LnormD{\tilde{f}^N - \tilde{p}^N}{\infty}{I_i^N} 
    = \LnormD{\sum_{j \in \mathcal{A}^N} (f - p_j^N) \, \Phi_j^N}{\infty}{I_i^N} \\
    &\qquad = \LnormD{\sum_{v \in \mathcal{V}_d} (f - p_{i + v}^N) \, \Phi_{i + v}^N}{\infty}{I_i^N} 
    + \sum_{\substack{j \in \mathcal{A}^N \\ j-i \notin \mathcal{V}_d}} \LnormD{f - p_j^N}{\infty}{I_i^N} \LnormD{\Phi_j^N}{\infty}{I_i^N} \\
    &\qquad \leq \left( \max_{v \in \mathcal{V}_d} \LnormD{f - p_{i+v}^N}{\infty}{I_i^N} \right) \LnormD{\sum_{v \in \mathcal{V}_d} \Phi_{i+v}^N}{\infty}{I_i^N} + N^d \mathcal{C}(d,k,s,f) \varepsilon \\
    &\qquad \leq \fr{\mathcal{C} (d,k,s,f)}{N^{s-k}}(1 + d\varepsilon) + N^d \mathcal{C}(d,k,s,f) \varepsilon \\
    &\qquad \leq \left( 1 + \fr{\delta}{3} \right) \fr{\mathcal{C} (d,k,s,f)}{N^{s-k}},
\end{aligned}\ee
with suitably chosen $\varepsilon$.

Whereas for $1 \leq k < s$, we have
\be\begin{aligned}
\label{fTildebyPolynomialInSobolev}
    &\WnormD{\tilde{f}^N - \tilde{p}^N}{k}{\infty}{I_i^N} 
    = \WnormD{\sum_{j \in \mathcal{A}^N} (f - p_j^N) \, \Phi_j^N}{k}{\infty}{I_i^N} \\
    &\qquad \leq \WnormD{\sum_{v \in \mathcal{V}_d} (f - p_{i + v}^N) \, \Phi_{i + v}^N}{k}{\infty}{I_i^N} 
    + \sum_{\substack{j \in \mathcal{A}^N \\ j-i \notin \mathcal{V}_d}} 2^k \LnormD{f - p_j^N}{\infty}{I_i^N} \WnormD{\Phi_j^N}{k}{\infty}{I_i^N}.
\end{aligned}\ee
To bound the first term, we invoke the general Leibniz rule to get
\be\begin{aligned}
    D^\beta \left( \sum_{v \in \mathcal{V}_d} (f - p_{i+v}^N) \, \Phi_{i+v}^N \right) \leq \sum_{\beta' \leq \beta} \binom{\beta}{\beta'} \sum_{v \in \mathcal{V}^d} \left| D^{\beta'} (f - p_{i+v}^N) \right| \left| D^{\beta - \beta'} \Phi_{i+v}^N \right|,
\end{aligned}\ee
with $x \in I_i^N$. Defining $l = |\beta - \beta'|$ for $\beta' \leq \beta$, the following holds for every $v \in \mathcal{V}_d$,
\begin{equation}
    \left| D^{\beta'} (f - p_{i+v}^N) \right| \leq \WnormD{f - p_{i+v}^N}{k-l}{\infty}{I_i^N} \leq \fr{\mathcal{C}(d, k-l, s, f)}{N^{s-k+l}},
\end{equation}
and, using Lemma \ref{tanhProperty} and the explicit value of $\alpha$ (in \eqref{valueOfAlpha}), we have
\begin{equation}
    \left| D^{\beta - \beta'} \Phi_{i+v}^N \right| \leq \alpha^l (2l)^{2l} = N^l (2l)^{2l} \max \left\{R^l, \ln^l \left( \fr{(2k)^{k+1} (Nk)^k}{e^k \varepsilon}\right) \right\}.
\end{equation}
Thus, the first term of \eqref{fTildebyPolynomialInSobolev} can be estimated as 
\begin{equation}
    \WnormD{\sum_{v \in \mathcal{V}_d} (f - p_{i + v}^N) \, \Phi_{i + v}^N}{k}{\infty}{I_i^N} \leq 2^k 3^d \fr{\mathcal{C}(d, k, s, f)}{N^{s-k}} (2k)^{2k} \max \left\{R^k, \ln^k \left( \fr{(2Nk^2)^{k+1}}{e^k \varepsilon}\right) \right\}, \nonumber
\end{equation}
which helps simplify \eqref{fTildebyPolynomialInSobolev} to
\be\begin{aligned}
    \WnormD{\tilde{f}^N - \tilde{p}^N}{k}{\infty}{I_i^N} &\leq 2^k 3^d \fr{\mathcal{C}(d, k, s, f)}{N^{s-k}} (2k)^{2k} \max \left\{R^k, \ln^k \left( \fr{(2Nk^2)^{k+1}}{e^k \varepsilon}\right) \right\} \\
    &\qquad + N^d 2^k \mathcal{C}(d,k,s,f) (2k)^{2k} N^k (k+1)^k \left(\fr{2Nk^2}{e} \right)^{k/2} \sqrt{\varepsilon} \\
    &\leq 3^d \left( 1 + \fr{\delta}{3} \right) (2(k+1))^{3k} \max 
    \left\{ R^k, \ln^k \left( \fr{2N k^2}{\varepsilon^\fr{1}{k+1} e} \right) \right\} \fr{\mathcal{C} (d,k,s,f)}{N^{s-k}},
\end{aligned}\ee
where we choose 
\begin{equation}
    \varepsilon \leq \fr{\delta^2}{N^{2s+2d+k} k^k},
\end{equation}
and assume that $0 < \delta < 5/6$.

{\bf{Step 5. Constructing the EML Tree.}}
In this step, our objective is to construct an EML tree $\tilde{p}_\theta^N$ such that it is exactly equal to the approximate global polynomial, i.e. $\tilde{p}_\theta^N = \tilde{p}^N$, and find an upper bound on the size and depth of such trees.

Keep in mind that, in view of the theorem statement, our goal here is to establish an approximation result on $(0,1]^d$. Accordingly, the EML tree constructed in this step is only defined on this domain, since zeros in the input may produce singularities through the logarithmic components appearing in the definition of the EML.

{We cannot directly appeal to Lemma \ref{lemma_multivariatePolynomial} or even Lemma \ref{lemma_multivariatePolynomial_with_decomposition} in constructing the EML tree representing $\tilde{p}^N$ because the individual polynomials $p_j^N$ (for $j \in \mathcal{A}^N$) could be negative, which then prevents us from feeding it into a subsequent multiplication block, whose first layer contains logarithms for both inputs. Thus, we appeal to a sign-based decomposition as follows: we rewrite $\tilde{p}^N$ as
\begin{equation}
\label{decomposition_final_EML_tree}
    \tilde{p}^N(x) = \sum_{j \in \mathcal{A}^N} p_j^N \Phi_j^N = \sum_{j \in \mathcal{A}^N} \left( p_{j, +}^N - p_{j, -}^N \right) \Phi_j^N = \sum_{j \in \mathcal{A}^N} p_{j,+}^N \Phi_j^N - \sum_{j \in \mathcal{A}^N} p_{j,-}^N \Phi_j^N,
\end{equation}
where, similarly to \eqref{polynomial_decomposition}, we define
\begin{equation}
    p_{j,+}^N (x) = \sum_{\substack{|\alpha| \leq s-1 \\ c_\alpha > 0}} c_{\alpha} x_1^{\alpha_1} x_2^{\alpha_2} \cdots x_d^{\alpha_d}, \quad \text{ and } \quad p_{j,-}^N (x) = \sum_{\substack{|\alpha| \leq s-1 \\ c_\alpha < 0}} |c_{\alpha}| x_1^{\alpha_1} x_2^{\alpha_2} \cdots x_d^{\alpha_d}.
\end{equation}
for $x \in (0,1]^d$.}

{Thus, for every $j \in \mathcal{A}^N$, we construct $p_{j,+}^N$ and $p_{j,-}^N$ through Lemma \ref{lemma_multivariatePolynomial}, obtaining the exact EML representations $p_{j,+,\theta}^N$ and $p_{j,-,\theta}^N$. And, for each $\Phi_j^N$, we can construct an associated EML tree $\Phi_{j,\theta}^N$ using Lemma \ref{lemma_partitionsOfUnityEML}. Each term $p_{j,+,\theta}^N \, \Phi_{j,\theta}^N$ has size 
\be\begin{aligned}
    N(p_{j,+,\theta}^N \, \Phi_{j,\theta}^N) &= \underbrace{(8d - 3) \binom{d+s-1}{s-1} - 3}_\text{for $p_{j,+,\theta}^N$} + \underbrace{(63d - 6)}_\text{for $\Phi_{j, \theta}^N$} + \underbrace{6}_\text{for $\times$} \\
    &=(8d - 3) \binom{d+s-1}{s-1} + 63d - 3,
\end{aligned}\ee
and similarly for $N(p_{j,-,\theta}^N \, \Phi_{j,\theta}^N)$.}

{We need $|\mathcal{A}^N| = N^d$ such terms and $N^d - 1$ addition blocks for each summation and a subtraction block to get $\tilde{p}_\theta^N$, thus yielding
\begin{equation}
\begin{aligned}
    N(\tilde{p}_\theta^N) &\leq \underbrace{2 N^d \left( (8d - 3) \binom{d+ s-1}{s-1} + 63d - 3 \right)}_{\text{individual terms}} + \underbrace{2 \times 3 (N^d - 1)}_{\text{sum $N^d$ terms}} + \underbrace{3}_\text{for $-$ block} \\
    &= 2N^d \left( (8d - 3) \binom{d+ s-1}{s-1} + 63d \right) - 3.
\end{aligned}
\end{equation}
The depth of such a tree can be found by noting that the longest path consists of passing through $p_{\{1,\dots, 1\}, \pm, \theta}^N$ or $\Phi_{\{1, \dots, 1\}, \theta}^N$ (depending on which is deeper), a multiplication block, $N^d - 1$ addition blocks, and a subtraction block, giving the following
\be\begin{aligned}
    \text{Depth}(\tilde{p}_\theta^N) &\leq \underbrace{4d + \max \left\{ 2 \binom{d+s-1}{s-1} - 4, 7 \right\}}_{\text{passing through $p_{\{1,\dots, 1\}, \theta}^N$ or $\Phi_{\{1, \dots, 1\}, \theta}^N$}} + \underbrace{4}_{\text{$\times$ block}} + \underbrace{2(N^d - 1)}_{\text{$+$ blocks}} + \underbrace{2}_\text{$-$ block} \\
    &= 4d + \max \left\{ 2 \binom{d+s-1}{s-1} - 4, 7 \right\} + 2N^d + 4,
\end{aligned}\ee
validating the upper bounds provided in the statement of the theorem.}

Evidently, since $\tilde{p}_\theta^N = \tilde{p}^N$, then on $(0,1]^d$, we have that
\begin{equation}
    \WnormD{\tilde{p}^N - \tilde{p}_\theta^N}{k}{\infty}{(0,1]^d} = 0.
\end{equation}

{\bf{Step 6. Conclusion.}}
Finally, for $k=0$, we can combine, via the triangle inequality, the previous estimates to obtain
\be\begin{aligned}
    \LnormD{f - \tilde{p}_\theta^N}{\infty}{(0,1]^d} \leq \left(1 + \fr{\delta}{3} \right) \fr{\mathcal{C} (d, 0, s, f)}{N^s},
\end{aligned}\ee
where we have chosen $\varepsilon$ to satisfy
\begin{equation}
\label{epsilonChoice}
    \varepsilon = \fr{\delta^2 \min \{1, \sqrt{\mathcal{C} (d,k,s,f)} \}}{N^{2s+2d+k} k^k 2^{(k+1)d} d \max \left\{1, \WnormO{f}{k}{\infty}^{1/2} \right\}}.
\end{equation}

Additionally, for $1 \leq k < s$, \eqref{epsilonChoice} implies that
\begin{equation}
    \varepsilon^\fr{1}{k+1} \geq \fr{\delta \min \{1, \sqrt{\mathcal{C} (d,k,s,f)}\}}{N^{s+d+1} k 2^d \sqrt{d} \max \left\{1, \WnormO{f}{k}{\infty}^{1/2} \right\}}.
\end{equation}
Therefore, we obtain the bound
\be\begin{aligned}
    \WnormD{f - \tilde{p}_\theta^N}{k}{\infty}{(0,1]^d}
        % &\qquad \leq \WnormO{f - \tilde{f}^N}{k}{\infty} + \WnormO{\tilde{f}^N - \tilde{p}^N}{k}{\infty} + \WnormO{\tilde{p}^N - \tilde{p}_\theta^N}{k}{\infty} \\
        % &\qquad \leq \WnormO{f - \tilde{f}^N}{k}{\infty} + \WnormO{\tilde{f}^N - \tilde{p}^N}{k}{\infty} \\
    &\leq 3^d (1+\delta) (2(k+1))^{3k} \max \left\{R^k, \ln^k \left( \beta N^{s+d+2} \right) \right\} \fr{\mathcal{C}(d,k,s,f)}{N^{s-k}},
\end{aligned}\ee
where we define $\beta$ to be
\begin{equation}
    \beta = \fr{k^3 2^d \sqrt{d} \max \left\{1, \WnormO{f}{k}{\infty}^{1/2} \right\}}{\delta \min \{1, \sqrt{\mathcal{C}(d,k,s,f)}\}}.
\end{equation}

This concludes the proof of Theorem \ref{main_theorem}.
\end{proof}

\section{Proof of Corollary \ref{main_corollary}}
\label{proofOfCorollary}
\begin{proof}
Let $\eta > 0$.
Choose $\gamma \in (0,1)$, to be specified later, and define $T_\gamma : [0,1]^d \to (0,1]^d$ by
\begin{equation}
    T_\gamma (x) \coloneqq (1 - \gamma)x + \fr{\gamma}{2} \mathbf{1}, 
\end{equation}
where $\mathbf{1} = (1, \dots, 1)$. Crucially, we observe that
\begin{equation}
    T_\gamma([0,1]^d)
    =
    \left[\frac{\gamma}{2},\,1-\frac{\gamma}{2}\right]^d
    \subset (0,1)^d.
\end{equation}

Moreover,
\begin{equation}
    \sup_{x \in [0,1]^d} |T_\gamma(x) - x| = \sup_{x \in [0,1]^d}
    \left|
        -\gamma x + \frac{\gamma}{2}\mathbf{1}
    \right|
    \leq C_d \gamma,
\end{equation}
so that $T_\gamma \to \mathrm{Id}$ uniformly as $\gamma \to 0$. In particular, since $f \in W^{s, \infty} ([0,1]^d)$ with $s > k$, for every multi-index $|\alpha| \leq k$, $D^\alpha f$ is Lipschitz. Hence,
\begin{equation}
    \LnormD{D^\alpha f - (D^\alpha f) \circ T_\gamma}{\infty}{[0,1]^d} \leq \LnormD{\na D^\alpha f}{\infty}{[0,1]^d} \, \LnormD{T_\gamma - \mathrm{Id}}{\infty}{[0,1]^d} \leq C_{f,k} \, \gamma.
\end{equation}
We can choose $\gamma > 0$ sufficiently small to get
\begin{equation}
    \WnormD{f - f \circ T_\gamma}{k}{\infty}{[0,1]^d} \leq \fr{\eta}{2}.
\end{equation}
By Theorem \ref{main_theorem}, there exists an $\operatorname{EML}_\theta$ tree $p_\theta$ such that
\begin{equation}
    \WnormD{f - p_\theta}{k}{\infty}{(0,1]^d} \leq \fr{\eta}{2}.
\end{equation}
We define the map $p_{\gamma, \theta}$ as
\begin{equation}
    p_{\gamma, \theta} (x) = p_\theta (T_\gamma(x)),
\end{equation}
which is well-defined on $[0,1]^d$. Since $T_\gamma([0,1]^d)\subset(0,1)^d$, we have, for any $g \in W^{k,\infty})[0,1]^d)$, that 
\begin{equation}
\|g\circ T_\gamma\|_{W^{k,\infty}([0,1]^d)}
\le
\|g\|_{W^{k,\infty}((0,1]^d)},
\end{equation}
and hence
\begin{equation}
\|f\circ T_\gamma - p_\theta\circ T_\gamma\|_{W^{k,\infty}([0,1]^d)}
\le
\|f - p_\theta\|_{W^{k,\infty}((0,1]^d)}
\le \frac{\eta}{2}.
\end{equation}

Finally, combining the previous estimates, we obtain
\be\begin{aligned}
    \WnormD{f - p_{\gamma, \theta}}{k}{\infty}{[0,1]^d} &= \WnormD{f - p_\theta \circ T_\gamma}{k}{\infty}{[0,1]^d} \\
    &\leq \WnormD{f - f \circ T_\gamma}{k}{\infty}{[0,1]^d} + \WnormD{f \circ T_\gamma - p_\theta \circ T_\gamma}{k}{\infty}{[0,1]^d} \\
    &\leq \fr{\eta}{2} + \fr{\eta}{2} = \eta.
\end{aligned}\ee
Naively, to construct $p_{\gamma, \theta}$ as an EML tree, we alter the inputs, since for each $i \in \{1, \dots, d\}$, we have to construct the map $x_i \mapsto (1-\gamma) x_i + \gamma/2$. One such map requires 2 EML atoms with appropriate parameters. {We have to construct such blocks for all used instance of the input variables in the EML tree (i.e. where their multiplicative prefactor is not set to $0$), since blocks in an EML tree are not reusable and so they cannot be fed into two distinct inputs.}

{We now count all used appearances of the inputs in the EML tree $p_\theta$, in order to substitute $x_i$ by $T_\gamma (x_i)$. For each $j \in \mathcal{A}^N$, we have to construct $p_{j, +, \theta}^N$, $p_{j, -, \theta}^N$, and 2 of $\Phi_{j, \theta}^N$. Additionally, we know that for each
\begin{itemize}
    \item $p_{j, \pm, \theta}^N$: there are $d \binom{d+s-1}{s-1}$ instances of the input variables in polynomial block,
    \item $\Phi_{j, \theta}^N$: there are 12 used appearances of the inputs in each $\rho_{j_i}^N (x_i)$ block (as seen in Figure \ref{fig:partitionOfUnity}), of which there are $d$ in each $\Phi_{j, \theta}^N$.
\end{itemize}
Thus, summing all these contributions together, we notice that we need $2N^d \left( d \binom{d+s-1}{s-1} + 12d \right)$ substitutions of $x_i \mapsto T_\gamma$, each requiring $2$ EML blocks. This produces the upper bound on $N(p_{\gamma, \theta})$ in \eqref{bound_for_corollary}.}

{On the other hand, the depth of $p_{\gamma, \theta}$ is only increased by $2$ EML blocks relative to $p_\theta$.}
% thus requiring an additional $2 \times 2 d N^d \binom{d+s-1}{s-1}$ EML atoms to the size of $p_\theta$, 

This concludes the proof of Corollary \ref{main_corollary}.
\end{proof}

\end{document}